 \newif\ifprivate\privatetrue %
\newif\ifuc\uctrue           %

\documentclass{article}

\usepackage[utf8]{inputenc} %
\usepackage[T2A,T1]{fontenc} %
\usepackage[russian,canadian]{babel}
\usepackage[style=american]{csquotes}

\usepackage[position,nonatbib,final]{neurips_2026}
\usepackage{url}            %
\usepackage{booktabs}       %
\usepackage{amssymb}       %
\usepackage{nicefrac}       %
\usepackage{microtype}      %
\usepackage{xcolor}         %
\usepackage{dsfont}

\usepackage{amsmath}
\usepackage{mathtools}
\usepackage{keytheorems}
\usepackage{geometry}
\usepackage{verbatim}
\usepackage[pdftex]{graphicx}
\usepackage{siunitx}
\usepackage{wrapfig}

\usepackage{hyperref}       %
\usepackage[noabbrev,capitalize]{cleveref}
\crefname{appendix}{Appendix}{Appendices}
\Crefname{appendix}{Appendix}{Appendices}

\usepackage{tikz}

\usepackage[style=alphabetic,natbib,maxbibnames=30,backend=biber,abbreviate=false,maxcitenames=2,useprefix=true]{biblatex}
\DefineBibliographyExtras{canadian}{\uspunctuation}
\renewbibmacro{in:}{}
\DeclareNolabel{%
  \nolabel{\string\p{Dash}}%
}

\ifprivate\def\todo#1{{\color[rgb]{.6,.3,.0}\it[[ToDo:\ #1]]}}\else\def\todo#1{}\fi %
\ifprivate\def\note#1{{\color[rgb]{.5,.5,.0}\it[[Note:\ #1]]}}\else\def\note#1{}\fi %

\newkeytheorem{theorem}
\newkeytheorem{definition}[sibling=theorem]
\newkeytheorem{proposition}[sibling=theorem]
\newkeytheorem{lemma}[sibling=theorem]
\newkeytheorem{corollary}[sibling=theorem]
\newkeytheorem{remark}[sibling=theorem]

\newcommand{\kl}[2]{D_\mathrm{KL}\left(#1\mathrel\Vert#2\right)}
\newcommand{\prob}{\mathbb P}
\DeclareMathOperator{\expect}{\mathbb E}
\let\E\expect
\newcommand{\dist}{\mathcal D}

\newcommand{\reals}{\mathbb R}
\newcommand{\nats}{\mathbb N}
\newcommand{\ints}{\mathbb Z}
\newcommand{\bits}{\mathbb B}

\newcommand{\starrel}[2]{%
  \mathrel{\smash{\mathchoice
    {\stackrel{#1}{#2}}%
    {\scriptstyle\stackrel{#1}{#2}}%
    {\scriptstyle\stackrel{#1}{#2}}%
    {\scriptscriptstyle\stackrel{#1}{#2}}%
  }}%
}
\newcommand{\eqp}{\starrel{+}{=}}
\newcommand{\leqp}{\starrel{+}{\le}}
\newcommand{\geqp}{\starrel{+}{\ge}}
\newcommand{\eqx}{\starrel{\times}{=}}
\newcommand{\leqx}{\starrel{\times}{\le}}
\newcommand{\geqx}{\starrel{\times}{\ge}}
\newcommand{\cF}{\mathcal F}
\newcommand{\cH}{\mathcal H}
\newcommand{\cW}{\mathcal W}
\newcommand{\cX}{\mathcal X}
\newcommand{\cY}{\mathcal Y}
\newcommand{\indic}{\mathds{1}}
\newcommand{\KM}{\mathit{KM}}
\newcommand{\MI}{\mathit{MI}}
\newcommand{\tp}{^\mathsf{T}}
\DeclareMathOperator{\Rad}{Rad}
\DeclareMathOperator{\VCdim}{VCdim}
\DeclareMathOperator*{\argmin}{argmin}

\newcommand{\xor}{\mathbin{\mathtt{xor}}}

\newcommand{\httpsurl}[1]{\href{https://#1}{\nolinkurl{#1}}}

\DeclarePairedDelimiter{\abs}{\lvert}{\rvert}
\DeclarePairedDelimiter{\norm}{\lVert}{\rVert}
\DeclarePairedDelimiter{\floor}{\lfloor}{\rfloor}

\DeclarePairedDelimiterX{\inner}[2]{\langle}{\rangle}{#1,#2}

\title{Understanding Generalization\\Requires Universal Induction}

\author{%
  Aram Ebtekar\\
  AIXI Labs \\
  \!\!\!\!\texttt{aramebtech@gmail.com}\!\!\!\! \\
  \And
  Marcus Hutter \\
  Google DeepMind \& ANU \\
  \httpsurl{www.hutter1.net} \\
  \And
  Danica J.\ Sutherland \\
  UBC \& Amii \\
  \texttt{dsuth@cs.ubc.ca} \\
}

\makeatletter
\if@preprint \renewcommand{\@noticestring}{Preliminary version as of \today. Comments welcome.} \fi
\makeatother

\begin{document}

\maketitle

\begin{abstract}
Classical statistical theory is insufficient to explain the successes of general-purpose AI models, because it depends on handcrafted inductive biases that it cannot justify. No Free Lunch (NFL) theorems force any learner that beats chance on some environments to underperform on others. We might hope that past experience informs which environments to expect, but NFL applies equally to meta-learning. Thus, any method that makes meaningful predictions necessarily begins with an inductive bias external to the data. Choosing to bias toward short programs yields Solomonoff induction (SI), whose performance is competitive against all computable learners -- albeit up to ``constants'' that become large when comparing against specialized methods that exploit background information. We therefore relativize SI to an \emph{information vantage point}, biasing toward short programs with access to all preexisting information. This reframes the inductive bias: instead of seeking some absolute notion of simplicity, we favour \emph{accessibility} with respect to our vantage point. An algorithm can only outpredict the relativized SI to the extent that its code contains additional information about the data, and no algorithm can generate such information. While SI is incomputable and hence not a practical algorithm, it provides a formal optimum for inference in the limit of infinite compute, and there is evidence to suggest that frontier AI systems roughly approximate it. Thus, the only known answer to meta-NFL is rooted in algorithmic information theory, which we should expect to play a fundamental role in explaining the generalization behaviour of modern (and future) AI systems.
\end{abstract}

\section{Introduction} \label{sec:intro}

Machine learning extrapolates from seen data to predict unseen data. Since deductive logic alone does not force any relationship between the seen and unseen data, such
\emph{inductive} inferences necessarily depend on structural assumptions about the data, called \emph{inductive bias}.
That the Sun has risen each day so far
does not logically imply that it will do so tomorrow,\footnote{%
One day, it may engulf the Earth as a red giant, or it may shrink into a white dwarf as the Earth drifts away \citep{Esseldeurs_2026}.
Neither situation is \emph{logically} necessary;
it remains deductively possible that the rules of gravity will change tomorrow.}
but we would find life difficult to navigate
without some confidence in tomorrow's sunrise.
This is often justified via Occam's razor, preferring simple explanations -- but what is a ``simple'' explanation, and why should we prefer one?

Theoretical analyses of learning problems often
assert some notion of smoothness or ``simplicity.''\footnote{For example, in one-dimensional density estimation,
one can assume there is an underlying density function
whose (say) third derivative is $1$-Lipschitz,
and then prove that an algorithm successfully estimates such densities
\citep[e.g.][Section 1.2]{tsybakov}.
}
In many settings, we might be satisfied to accept that real data should be simple in a particular sense.
As we apply machine learning to increasingly general problems, however,
the required assumptions become more complex and less \emph{a priori} defensible.
To obtain a theoretical account of why learning works,
one must argue for a particular choice of inductive bias.

This choice is commonly justified by pointing to earlier instances in which the proposed inductive bias was successful. In other words, a part of the learning algorithm is itself inductively meta-learned \citep{schmidhuber1987evolutionary,thrun-pratt,hospedales2022meta} from a lifetime of experience; in turn, our natural lifelong learning algorithm is informed by billions of years of Darwinian evolution \citep{naoki-peter:olfaction,goldfeder2026ai}. 
This suggests a promising case study: since it seems implausible that our microbial ancestors billions of years ago ``knew'' much, we can use evolution as a model of learning from a state of true ignorance.

Naive attempts to build a theory of learning from true ignorance, however, run into No Free Lunch theorems.
In statistical learning terms,
we have no \emph{a priori} way to guarantee low generalization error.
A model class that satisfies uniform convergence can ensure
that we do not substantially overfit (i.e.\ find a predictor whose error is much higher on fresh data than on training data),
but cannot guarantee
that we do not underfit (i.e.\ have high training error).
Whether such explanations of non-overfitting even apply to
practical neural networks
remains unclear after much investigation;
there are meaningful negative results \citep{neyshabur:inductive,zhang:rethinking,vaish:uniform-failures,jiang:fantastic,gastpar:no-fantastic},
but also some positive indications
\citep{dziugaite2017computingnonvacuousgeneralizationbounds,zhou:nonvac-imagenet,yang:exact-gap,zhou:moreau,lotfi2024non,lotfi2024unlocking}.
If they do apply, why should we expect practical neural networks not to underfit a new problem?
Alternatively, ``soft-preference'' techniques
\citep[see][]{wilsonposition},
such as structural risk minimization \citep{VapnikChervonenkis1974,Vapnik1991PRM} in a universal hypothesis class,
can ensure that the model class will not underfit. The drawback is that they cannot
guarantee non-overfitting
unless the soft preference -- the algorithm's inductive bias -- happens to be suitable for the problem. It seems we can provably prevent overfitting or underfitting, but not both.
Despite No Free Lunch, it is an empirical fact that life and machines alike are often able to learn well.
One might ask: why care about theory?
\Citet{kawaguchi:generalization}
identify three ``practical roles''
for generalization theory:
to provide
(1) guarantees on expected risk,
(2) guarantees on the generalization gap,
and (3) insights to guide model selection.
Learning theory addresses role (2)
for ``small'' hypothesis classes,
showing that test error is near training error.
(The use of a validation set can, in some ways, make any class effectively small \citep[Section 4]{kawaguchi:generalization}.)

No Free Lunch theorems appear to undermine roles (1) and (3), requiring that we assume %
some specialized structure \emph{a priori}. Even when the assumed structure holds, it provides no recipe for systematically exploiting \emph{all} available structure, giving very loose bounds \citep{dziugaite2017computingnonvacuousgeneralizationbounds,jiang:fantastic}. Practical methods based on such assumptions have not scaled well.
\Citet{sutton2019bitter}
argues ``we have to learn the bitter lesson that building in how we think we think does not work in the long run,'' observing e.g.\ that neural networks using ``only the notions of convolution and certain kinds of invariances'' greatly outperform predecessors based on SIFT features \citep{sift}.
Since his essay, convolutional networks have largely given way to methods with even weaker built-in structure \citep{vit,mlp-mixer,dinov2}.

In \cref{sec:nfl}, we expand on how No Free Lunch theorems preclude the possibility of learning without an inductive bias -- and, in fact, a meta-inductive bias, which itself cannot be learned via experience nor evolution \citep{wolpert2023implications}.
Today, we see generalist AI models deployed on increasingly arbitrary tasks.
What sort of inductive bias would account for their generalization performance on new problems?

With the advent of the attention mechanism \citep{vaswani2023attentionneed} and reasoning modes \citep{openai2024learningtoreason,deepseek-r1,snell2025scaling},
leading model classes are rapidly moving toward supporting general, even universal forms of computation \citep{attn-turing-complete,expressive-cot,pencil,li:constant-size-turing}.
Solomonoff induction  \citep{solomonoff1964formal} is a logical extreme in this direction,
corresponding to Bayesian inference with a weighted mixture prior over all possible computer programs. Due to the halting problem \parencites{turing1937computable}[Section 4.2]{sipser}, Solomonoff induction is not computable, but it provides an idealization of learning in the limit of unbounded computation; it is a central object in algorithmic information theory \cite{li2019introduction,hutter2024introduction}.
However, it necessarily chooses prior weights for each of the infinitely many possible computer programs. What justifies that prior?

Inspired by \citet{schurz2008meta,schurz2019hume},
we argue in \cref{sec:accessibility} that while there is no absolute winner among the class of all mathematically well-defined predictors, most of these predictors are inaccessible to us as practical algorithms.
This motivates an inductive bias in favour of \emph{accessibility} from a given \emph{information vantage point}, which \cref{sec:solomonoff} formalizes as a universal computer equipped with an oracle.
We prove that Solomonoff induction from this vantage point is optimal, in the sense that the only way to do better is to hardcode additional information about the dataset. Formally, \cref{thm:oraclesolregret} says that oracle Solomonoff induction's regret
against any predictor is bounded by that predictor's
algorithmic mutual information with the data -- the extent to which it ``knew the answer'' in advance.

In \cref{sec:nn-gen},
we discuss preliminary evidence supporting this point of view as a reasonable 
idealized description of modern AI systems.
\Cref{sec:infproperties,sec:osiproperties} cover additional details of oracle Solomonoff induction, \cref{sec:nfl-slt} discusses what some key results in statistical learning theory (do not) say in relation to No Free Lunch,
and \cref{sec:boltmzann-brains} gives implications for the problem of Boltzmann brains in cosmology.
Overall, we argue that \textbf{the only known answer to the meta-No Free Lunch problem is based on algorithmic information theory, and that this point of view gives a viable strategy for explaining the generalization behaviour of modern (and future) AI systems}.
\section{No Free Lunch in learning theory}
\label{sec:nfl}

No Free Lunch theorems have an extensive history, dating back to Hume's problem of induction and Goodman's new riddle of induction \cite{hume1748enquiry,putnam1963degree,goodman1983fact,wolpert1996lack,sloman2005problem,adam2019no,schurz2019hume,belot2020absolutely,sterkenburg2021no,wolpert2023implications}. Here, we focus on a straightforward formulation for finite input and output sets, whose proof is immediate by counting. Let $\ints_n:=\{0,1,\ldots,n-1\}$.

\begin{proposition}
\label{lem:nfl}
For a target function $f^*:\ints_n\rightarrow\ints_k$, let $D_\mathrm{train}\subset\ints_n$ be a set of $|D_\mathrm{train}|=m<n$ observations. Then, there are exactly $k^{n-m}$ functions $f:\ints_n\rightarrow\ints_k$ consistent with these observations, in the sense that $f(x)=f^*(x)$ for all $x\in D_\mathrm{train}$.
At every unobserved $x'\in\ints_n\setminus D_\mathrm{train}$, for every $y' \in \ints_k$, exactly $k^{n-m-1}$ of the functions consistent with $f^*$ on $D_\mathrm{train}$ have $f(x')=y'$.
\end{proposition}

As a consequence of \Cref{lem:nfl}, no matter how a prediction algorithm uses the training dataset $(x,f^*(x))_{x\in D_\mathrm{train}}$ to predict $f^*(x')$, it will fail for most possible target functions $f^*$.\footnote{Note that this applies to all algorithms that depend only on the training data, regardless of how well they fit it, because such algorithms cannot distinguish between the $k^{n-m}$ target functions that produce the same training labels.} To better appreciate the situation, let us examine how Bayesian and frequentist inference -- representing two major schools of thought in statistical methodology -- each fall short without an inductive bias.

\paragraph{Bayesian and frequentist no free lunches}
To avoid inductive bias, suppose the Bayesian statistician begins with a uniform prior over all $k^n$ possible functions. After observing the training dataset $(x,f^*(x))_{x\in D_\mathrm{train}}$, the posterior belief is a uniform distribution on the $k^{n-m}$ functions consistent with the data. No generalization is achieved, because the belief on non-training points $x'\notin D_\mathrm{train}$ remains uniform, assigning probability
${k^{n-m-1}} / {k^{n-m}} = {1}/{k}$
to each possible value of $f^*(x')$.

Similarly, suppose the frequentist statistician samples a random hypothesis $\hat f$ uniformly\footnote{Alternatively, one could choose an ``arbitrary'' consistent hypothesis, which might perform better (or worse) than the random choice. Typical frequentist analyses therefore assume either a random choice, a \emph{worst-case} choice (naturally only worse), or perhaps some particular tie-breaking rule (which would correspond to choosing some non-uniform inductive bias).}
among the $k^{n-m}$ functions that satisfy $\hat f(x)=f^*(x)$ for all $x\in D_\mathrm{train}$, to test on a hold-out set $D_\mathrm{test}\subset\ints_n\setminus D_\mathrm{train}$. Since $(\hat f(x))_{x\in D_\mathrm{test}}$ is uniformly distributed, the test rejects $\hat f$ with overwhelming probability. We can follow up by sampling new hypotheses to test, but if all we do is sample at random until we obtain satisfactory performance on $D_\mathrm{test}$, the resulting hypothesis will fail to generalize, remaining uniformly distributed on all $x'\notin D_\mathrm{train}\cup D_\mathrm{test}$.

We note an important difference between the two methodologies. The Bayesian statistician sets an explicit inductive bias at the start, via a prior. Thereafter, Bayes' rule determines optimal posterior inferences. This makes Bayesian methods readily automatable, perhaps contributing to their historical popularity in AI research. In contrast, the frequentist statistician delivers an inductive bias in an online fashion, providing guidance in the form of hypotheses to test one after another. This is more amenable to experimental science: when human scientists are unable to encode their beliefs into an explicit prior, they benefit from the ability to make a sequence of guesses.

In these naive forms, neither methodology is self-contained. Useful inferences require a nonuniform choice of Bayesian prior or hypothesis generation, but the theory provides no guidance on this choice, effectively outsourcing it to human design.

\paragraph{A concrete example}
\newcommand{\n}{{19}}
\newcommand{\nmthree}{{16}}
\newcommand{\nmfour}{{15}}
Let $n = k = \n$, so that we want to learn a function $f^*:\ints_\n\rightarrow\ints_\n$. We observe $f^*(4)=4$, $f^*(8)=8$, and $f^*(11)=11$. What is $f^*(0)$? Without further constraints, it can be anything: by \Cref{lem:nfl}, there are $\n^{\nmthree}$ (about 288 quintillion) functions compatible with the data, with exactly $\n^{\nmfour}$ (15 quintillion) of them assigning each possible value in $\ints_\n$ to $f^*(0)$.

Nonetheless, it feels intuitive to guess the function $f_1(x):=x$, which predicts $f_1(0)=0$. What makes $f_1$ a better guess than, say,
\[f_2(x)
:= x + (x-4)(x-8)(x-11)
= x^3 - 4x^2 - 6 x + 9
\pmod{\n},\]
which also matches the data but predicts $f_2(0)=9$? We might think of $f_2$ as representing a weaker inductive bias, perhaps being derived by fitting the data to a large hypothesis class
\[\cH_0:=\left\{x \mapsto a_3x^3+a_2x^2+a_1x+a_0 \pmod{\n} \,:\, a_i\in\ints_\n\right\}.\]
$\cH_0$ contains $\n^4 = 130{,}321$ elements, enough to \emph{overfit} our small dataset. Hence, observing $f_2(x)=f^*(x)$ for $x\in D_\mathrm{train}=\{4,8,11\}$ should not give us confidence that $f_2$ will continue to fit future data.
Since $f_1$ belongs to the much smaller hypothesis class
\[\cH_1:=\left\{x \mapsto a_1x+a_0\pmod{\n} \,:\, a_i\in\ints_{\n}\right\},%
\]
of size only $\n^2 = 361$,
we expect it not to overfit as much. The issue with this analysis is that $f_2$ is a member of not only $\cH_0$, but also some smaller classes,
such as
the following, also of size $361$:
\[\cH_2:=\left\{ x \mapsto x^3-4x^2+a_1x+a_0\pmod{\n} \,:\, a_i\in\ints_{\n}\right\}
.\]

Choosing $\cH_2$ may feel like ``cheating,'' since we constructed it by taking the leading terms $x^3-4x^2$ from $f_2$, which depended on the particular observations on $D_\mathrm{train}$.%
\footnote{We can even imagine choosing (as do many students first exposed to learning theory) the singleton hypothesis class $\{ f_2 \}$.}
But, regardless of how we the authors selected it,
what is it
for you the reader
that makes $\cH_2$ ``feel tailored'' to the training data
while $\cH_1$ does not? If we remove our intuitive preference for $f_1$ and $\cH_1$,
the situation is formally symmetric. If the true function is $f_1$, then fitting to $\cH_1$ performs well while fitting to $\cH_2$ performs poorly;
if the true function is $f_2$, the situation is reversed.

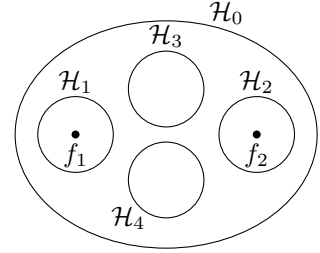
\begin{wrapfigure}[19]{r}{0.3\textwidth}
  \centering
  \vspace*{-3ex}
  \begin{tikzpicture}
    \fill (-1.2,0) circle (1.5pt);
    \node[below] at (-1.2,0) {$f_1$};
    \draw (-1.2,0) circle (0.5);
    \node at (-1.2,0.7) {$\mathcal{H}_1$};

    \fill (1.2,0) circle (1.5pt);
    \node[below] at (1.2,0) {$f_2$};
    \draw (1.2,0) circle (0.5);
    \node at (1.2,0.7) {$\mathcal{H}_2$};

    \draw (0,0.6) circle (0.5);
    \node at (0,1.3) {$\mathcal{H}_3$};

    \draw (0,-0.6) circle (0.5);
    \node at (-0.5,-1.1) {$\mathcal{H}_4$};

    \draw (0,0) ellipse (2 and 1.5);
    \node at (0.8,1.6) {$\mathcal{H}_0$};
  \end{tikzpicture}
  \caption{Both the ``simple'' function $f_1$ and the ``complex'' function $f_2$ belong to small hypothesis classes.
  The same can be said of any function in the large class $\cH_0$.
  To generalize well, we should choose a class that is small while containing (a good approximation of) the true function.}
  \label{fig:circles}
\end{wrapfigure}
We visualize the space of possible functions in \cref{fig:circles}.
Learning theory \citep{vapnik-chervonenkis-1968,valiant:pac,ssbd,mrt,ltfp}, roughly,
says that learning from a
``small'' hypothesis class (such as $\cH_1$ or $\cH_2$)
probably yields hypotheses with a small generalization gap:
their training performance is about the same as their test performance.
Large classes, such as $\cH_0$, lack this property:
when using them, we may overfit the training data.
Small classes are vulnerable to a different problem: if they contain no good approximation to the true function,
then we \emph{underfit},
i.e.\ the training performance is itself poor.
Indeed, $\cH_2$ underfits most random samples (of three or more points) labelled by $f_1$, and conversely $\cH_1$ underfits most samples labelled by $f_2$.

Learning theory recommends choosing a small class to prevent overfitting, but does not tell us \emph{which} small class prevents \emph{underfitting}. No Free Lunch says that no algorithm generalizes well on most functions, so no hypothesis class simultaneously satisfies both constraints for most functions. We are forced to favour certain kinds of functions \emph{a priori}. Why do we feel inclined to choose $\cH_1$ over the myriad equally small classes such as $\cH_2$? Low polynomial degree seems like a reasonable heuristic, but is not a general rule; it is disastrous for fitting other ``simple'' functions such as $f_3(x) := 10\floor{\frac{x}{10}}$.

\paragraph{Meta-learning an inductive bias}
Vitally, past experience cannot serve as a root justification for inductive biases:
even assuming (generously) that we identify the full labelling function for every problem we encounter,
meta-learning (learning to learn) is itself a learning problem \citep{wolpert2023implications}. Hence, past learning problems tell us nothing about what the next function might be.

\begin{corollary}[Meta-No Free Lunch] \label{thm:meta-nfl}
    Take $N$ prediction problems, with the $i$th revealing a function $g_i:\ints_n \to \ints_k$.
    After observing $g_1,\ldots,g_m$ with $m<N$, there remain $k^{(N-m)n}$ consistent meta-mappings $i\mapsto g_i$, exactly $k^{(N-m-1)n}$ of which predict each of the $k^n$ possible choices for $g_{m+1}$.
\end{corollary}
\begin{proof}
    Uniquely label each function $g:\ints_n \to \ints_k$ by a number $\operatorname{enc}(g)\in\ints_{k^n}$.
    Now consider the problem of
    learning the meta-function $f:\ints_{N}\rightarrow\ints_{k^n}$ given by $f(i) := \operatorname{enc}(g_{i+1})$.
    The result follows from applying \cref{lem:nfl}
    with $n,k,D_\mathrm{train},x'$ replaced by $N,k^n,\ints_m,m$, respectively.
\end{proof}

As suggested in \cref{sec:intro}, let's examine how Nature might learn $f^*=f_1$ in a toy model of Darwinian evolution. Imagine a population of creatures, each genetically wired to compute some function $f:\ints_{\n}\rightarrow\ints_{\n}$. For simplicity, we assume an infinite population, providing enough initial variation so that no mutations are needed. On each day, the creatures are tested on a single input $x$; a creature $f$ survives iff $f(x) = f_1(x) := x$. On the first three days, suppose the inputs are $4,8,11$. Then, all surviving creatures satisfy $f(4)=4,f(8)=8,f(11)=11$. Do the creatures adapt under selection? Specifically, given a new input on day four, is the survival rate better than chance?

The answer depends on the initial population. This setting is formally identical to our Bayesian analysis above, with natural selection taking the role of Bayesian updates from a prior (an initial population) to a posterior (a subsequent population).
If the initial creatures are uniformly sampled from the set of all functions, then the generalization performance -- the survival rate on day four -- is only as good as chance.
Without an explicit nonuniform bias at the outset,
we can never learn one.

\section{Accessibility as a universal inductive bias} \label{sec:accessibility}
In any realistic evolutionary setting, the initial distribution would not be uniform over all possible functions. In biology, perhaps even in the primordial soup where life began, molecules can arrange themselves into simple computing machines \citep{johnston2022symmetry}. Thus, functions that can be computed by simple circuits, such as $f_1(x):=x$ or $f_3(x):=10\floor{\frac{x}{10}}$, are sampled more frequently than others. This results in good generalization performance when testing on such functions.

The inductive bias toward simplicity is commonly known as Occam's razor, after the medieval philosopher William of Ockham.
Its mathematical formalization as Solomonoff induction was proposed by \citet{solomonoff1964formal}, and improved by \citet{levin1971thesis}. \citet{rathmanner2011philosophical} advocate for it as a ``gold standard'' for induction, although this view is challenged by \citet{mcgregor2014natural,herrmann2020pac,neth2023dilemma,sterkenburg2026solomonoff}.

One major issue is that when simplicity is measured by description length, it depends on the choice of description language.
Indeed, \emph{any} prior distribution $\nu(x)$ can be interpreted as a simplicity bias, since
a self-delimiting code of length $\lceil\log\frac{1}{\nu(x)}\rceil$ is shortest at the peaks of $\nu$. Even Solomonoff's \emph{universal} measures can be distributed arbitrarily on sequences of any prescribed finite length \citep{leike2015bad}. Thus, 
``simplicity'' appears to be a catch-all term for any possible preference we have.

Instead of seeking simplicity in absolute terms, \citet{schurz2008meta,schurz2019hume} argues that reality forces a meta-inductive bias toward inductive methods that are \emph{accessible} to us.
A method may be accessible because it is simple (e.g.\ ``always predict \texttt{1}'')
or established (e.g.\ 
``run an online AutoML system to do evolutionary hyperparameter search over neural network architectures optimized by Adam, and use the predictor with the best cross-validation score'').
A method may be \emph{inaccessible} because it requires a specific random seed, or hardcodes unseen data (e.g.\ the test set or next year's sporting results).
We would like to ensure our predictions are not much worse than any methods we could plausibly use (i.e.\ have low regret against all accessible predictors \citep{CesaBianchi_prediction_2006});
we are unbothered if our predictions are worse than those of a clairvoyant method we could never implement.\footnote{%
For the running example in \cref{sec:nfl}, we can easily write an online prediction algorithm that selects $f_1$ after observing only the first example $f^*(4)=4$. In contrast, we have no accessible means of specifically selecting $f_2$ until after observing all three examples, as that would require hardcoding information equivalent to these examples. Most of the functions agreeing on $D_\mathrm{train}$ are even worse, appearing random to us on $\ints_\n\setminus D_\mathrm{train}$ because no accessible computation would find any pattern.}

Learning algorithms are themselves defined by (pseudo)code, which is a form of information.
In general, we measure the inaccessibility of a piece of information (i.e.\ a string) by the length of the shortest ``generalized pointer'' (i.e.\ another string) to it.
For example, some strings that point to the text of Shakespeare's \emph{Hamlet} include, in order of decreasing length:
(1) the full text verbatim;
(2) the text in compressed \texttt{gzip} format, if the reader has e.g.\ \texttt{gzcat};
(3) the command ``\texttt{curl -L }\httpsurl{flgr.sh/txtfssHamtxt}'', if the reader has a Unix shell with Internet access;
(4) the string ``Text of Hamlet'', if the reader has memorized the play or knows where to find a copy.
Clearly, the accessibility of information is highly dependent on context. We refer to this context, representing the totality of information in our possession, as our \emph{information vantage point}.

For concreteness, take the point of view of an AI program that lives inside a computer, interfacing with the outside world via a Bash shell. Its vantage point is defined by the Bash interface, alongside connected storage and network devices (which include the AI's own code). As a result, some functions are accessible to the AI via short Bash commands
(e.g.\ simple calls to installed software libraries, or pretrained language models downloadable from HuggingFace),
while others are not
(e.g.\ someone else's private files, or next year's stock market movements).

The AI does not choose its own vantage point; rather, some engineer built the system, effectively determining its vantage point. If we assume a physical Church-Turing thesis, essentially saying that all physical processes are computable \citep{wolpert2024implications}, we can view the engineer as an algorithmic system with her own exogenous information vantage point. To maximize the AI's predictive ability, the engineer can give it access to all of her information, some of which may happen to be usefully correlated with the unseen test data; she cannot, however, \emph{create} additional useful information (\cref{thm:noinfoexnihilo}).

How can we reason about what the engineer can, or should, do?
If she is an algorithmic system, the outcome is predetermined \citep{lewis1981we}.
Yet if we compare to \emph{any} program she counterfactually could have written,
then she ``should'' have written one that hardcodes the test data.
As a middle ground, we consider a method accessible if it can be described by a small program.
In order to model maximal intelligence, we run the program with unbounded time and space, enabling it to search for and retrieve larger codebases and datasets via the interface that represents the information vantage point.

Short programs exactly capture the set of predictors we can obtain with substantial probability by typing at random. One might object that specific long programs are also obtainable with substantial probability using biased sampling, algorithmic search, or preexisting knowledge. However, all such cases are already covered by short programs in our Bash terminal, because biased or searchable samples can be compressed, and preexisting information can be downloaded using short commands.

Similar reasoning applies to Nature's learning via Darwinian evolution.
In the primordial soup, the most accessible functions are those that occur most frequently during natural processes that stochastically rearrange molecules
\citep{deutsch2013computation,ozkural2015ultimate,janzing2018does,ebtekar2025foundations}. As evolution selects for genetic information relevant to survival, subsequent generations start from a richer information vantage point.
In addition to this genetically hardwired information,
we further enrich our vantage points through education, life experience, cultural context, and access to various repositories of information.

\cref{sec:algprob} will formalize the information vantage point as a universal computer $U$ equipped with an oracle $q$, which can be queried for background information. If $U$ represents a Bash shell on a computer, and $q$ represents a snapshot of the local storage devices and World Wide Web, then a short program $p$ can download a Python interpreter, packages, open-source repositories, and datasets, and adapt them with a custom Python script, to make a state-of-the-art predictor.\footnote{%
Strictly speaking, some of this code is unnecessary: the fact that AI researchers and engineers wrote it using mathematical reasoning and old datasets shows that it is in principle derivable from a poorer vantage point, such as one from the year 1960.
The added contribution of our modern vantage point, which a learner from 1960 should not be expected to compete against, includes chaotic (unpredictable, even in principle)
trends in arts, culture, and politics, as well as information from new sensing instruments, such as astronomical data from modern telescopes.
The transformer architecture is readily accessible from our vantage point. From the vantage point of 1960 or ancient history, it appears inaccessible; however, given infinite computation to search for good predictors, perhaps even Plato could have discovered the transformer. It remains an open question how much
of the architecture is a historical accident, and how much follows from the generic task of natural language prediction. Future work may benefit from studying time-constrained versions of accessibility; see \citep{garrabrant2017,Xu2020A,finzi2026entropy,voitovych2026learning}.
} Thus, the central issue of which description language to adopt is resolved by granting access to all known languages via $q$.\footnote{In the spirit of No Free Lunch theorems, we should note that more data is not always better. The legendary Library of Babel \citep{borges:babel} is said to contain all possible 410-page books. This gigantic library is useless for the purposes of accessibility, however, since a set of directions pointing to one of its books would be as complex as the actual text
(unless the library is designed to privilege certain books, e.g.\ placing \emph{Hamlet} on an altar while leaving most other books on unremarkable shelves).
There is a short program that prints the entire Library of Babel, including the book consisting of the first 410 pages of our test data, but that book by itself need not be the output of any short program.}

\section{Solomonoff induction} \label{sec:solomonoff}

We now develop the formal theory of fully general optimal online learning with respect to an information vantage point. Our results are self-contained, deviating from standard presentations only where needed to make the vantage points explicit.

\subsection{Bayesian mixture predictors}
\label{sec:bayesian}

To balance simplicity with generality, we focus on optimizing the log loss (also called cross-entropy) in the setting of online binary sequence prediction. We make no distributional assumptions: the observation sequence $x\in\bits^*$ ($\bits:=\{\texttt{0},\texttt{1}\}$) may be chosen by an adversary.\footnote{Non-binary tokens and even continuous parameters are covered as special cases. For example, $x$ may be formed by concatenating self-delimiting tokens, sampled i.i.d. from a continuously parametrized distribution \citep[Section 3.7.2]{hutter2004universal}.}
We denote the length of $x$ by $|x|$, its $t$th bit by $x_t$, and its first $t-1$ bits by $x_{<t}$. A \emph{predictor} $\nu$ assigns next-bit probabilities $\nu(x_t\mid x_{<t})\ge 0$ at each step $t$, forming the measure
$\nu(x) := \prod_{t=1}^{|x|}\nu(x_t\mid x_{<t})$.
If $\nu(\texttt{0}\mid x) + \nu(\texttt{1}\mid x) = 1$
for all $x \in \bits^*$,
then $\nu$ is a \emph{probability measure}.
To allow prediction by general programs that may eventually cease to output bits
\parencites[Chapter 4]{li2019introduction}{wyeth2025value},
we weaken the requirement to
$\nu(\texttt{0}\mid x) + \nu(\texttt{1}\mid x) \le 1$,
and say $\nu$ is a \emph{semimeasure}
(short for \emph{semiprobability measure}). %
The log loss of a predictor $\nu$ over the sequence $x\in\bits^*$ is
\begin{equation}
\label{eq:loss}
L(\nu,\,x)
\,:=\, \log\frac{1}{\nu(x)}
\,=\, -\sum_{t=1}^{|x|}\log\nu(x_t\mid x_{<t}).
\end{equation}

Say we have a countable set of ``expert'' predictors $\{\nu_i\}_{i=1}^\infty$, with prior weights $w_i > 0$ (representing e.g.\  $\nu_i$'s prior credibility or accessibility), $\sum_i w_i \le 1$.
We form the mixture-of-experts predictor
\begin{equation}
\label{eq:bayesmixture}
\xi(x)
:= \sum_{i=1}^\infty w_i \nu_i(x)
,\qquad\text{for which}\;
L(\xi,\,x) - L(\nu_i,\,x)
= \log\frac{\nu_i(x)}{\xi(x)}
\le \log\frac{1}{w_i}.
\end{equation}
That is, the mixture $\xi \ge w_i \nu_i$ has low regret against (i.e.\ its loss is not much more than) every expert $\nu_i$.
This bound is independent of $x$, but depends on the weight $w_i$ assigned to each expert.

We can make a universal predictor by taking each possible program as an expert. 
Unfortunately, the halting problem precludes
this mixture (or any good approximation) from being computable.
Indeed, every computable predictor can suffer infinite regret against the predictor, computable in about the same time, that makes the opposite prediction with high confidence \parencites{putnam1963degree}[Section 3]{solomonoff2009algorithmic}{sterkenburg2026solomonoff}.
We could compete against the latter by adding it to our mixture \labelcref{eq:bayesmixture}, but systematically adding many such experts drastically increases the computation time, while still allowing infinite regret against an analogous predictor.
The joint optimization of time and sample efficiency presents a tradeoff \citep{schmidhuber2002speed,veness2011monte,filan2016loss,nakkiran2021turing,finzi2026entropy,voitovych2026learning}
whose analysis is beyond our scope.

Instead, we can already gain insights from idealizations \citep{weisberg2007three,gabaix2008seven} that optimize sample efficiency in the infinite-time limit. Several constructions of limit-computable predictors are known to dominate all computable experts \parencites{Schmidhuber2002Hierarchies}[Section 2.4.3]{hutter2004universal}{hutter2007semimeasures}{solomonoff2008three}, and may therefore be considered forms of ``universal induction.'' They appear to have similar properties, so for simplicity we focus on the original Solomonoff-Levin choice of $w_i$ and $\nu_i$ in \eqref{eq:bayesmixture}
\citep{solomonoff1964formal,levin1971thesis,solomonoff1978complexity}.
To introduce it, we need some concepts from algorithmic information theory \citep{li2019introduction,hutter2024introduction}.

\subsection{Elements of algorithmic probability}
\label{sec:algprob}

Computers encode data and algorithms alike as binary strings. Given a finite string $x\in\bits^*$, and a finite or infinite string $y\in\bits^*\cup\bits^\infty$, their concatenation $x_1x_2\ldots x_{|x|}y_1y_2\ldots$ is denoted $xy$, and we write $x\sqsubseteq xy$ to say that $x$ is a prefix of $xy$. For $x,y\in\bits^\infty$, their interleaving $x_1y_1x_2y_2\ldots$ is denoted $x\oplus y$. The infinite string of zeros is denoted $\mathbf 0:=\texttt{0}^\infty$. Natural numbers $n\in\nats$ are encoded as self-delimiting strings $\bar n\in\bits^*$ with $|\bar n|=O(\log n)$, so that e.g.\ each pair $(n,x)\in\nats\times\bits^*$ is uniquely decodable from the string $\bar n x$.\footnote{One such construction: $\bar n$ consists of $\floor{\log_2 n}$ \texttt{0}s, followed by $n$'s standard binary encoding.}

It will be convenient to fix a universal reference computer $U$, which we take to be a \emph{monotone Turing machine} with four binary data tapes. It has a pair of read-only
input tapes, one of whose contents we call an \emph{oracle} and the other a \emph{program}; a two-way read-write work tape; and a one-way write-only output tape.\footnote{For readers unfamiliar with the Turing machine formalism, it will suffice to think of $U$ as a computer that reads from two input streams and writes to one output stream.
Once a bit is written, it can never be overwritten.
The work tape functions as the working random-access memory (RAM).
While many authors define the input tapes as one-way,
this choice is not vital;
one-way inputs can be copied to the work tape,
or two-way input tapes can be prefixed with infinitely many zeros.
}
We can think of the oracle tape as containing a downloaded snapshot of all available data, e.g.\ the World Wide Web. Alternatively, it may be more intuitive to view the oracle as a hardware abstraction layer that fetches data on demand.\footnote{Our oracle takes integer addresses instead of Bash commands containing URLs, but it is easy to translate between the two.} For simplicity, we treat the data as fixed and ignore time and space complexity, making these two interpretations equivalent.

The tapes are infinitely long; this makes it possible, for example, to concatenate a randomized algorithm with its own infinite source of random bits, even though $U$ is deterministic. Every cell of the work tape is initially set to $\texttt{0}$. Given an oracle $q\in\bits^\infty$ and program $p\in\bits^\infty$, let $U^q(p)\in\bits^*\cup\bits^\infty$ denote the output of $U$, defined as the (finite or infinite) sequence of bits that it eventually writes. Since the output tape is one-way, this is well-defined regardless of whether $U$ halts.

We choose $U$ to be \emph{universal} in the sense that there exists an effective enumeration of all such monotone Turing machines $\{T_i\}_{i\in\nats}$, such that for all $i\in\nats$ and $p,q\in\bits^\infty$,
\begin{equation}
\label{eq:universalturing}
U^q(\bar i p) \,=\, T_i^q(p)
.\end{equation}
We can think of $\bar i$ as an interpreter for $T_i$ in the language of $U$. Note that when $q$ contains code for an interpreter, there exists a computer $T_i$ that seeks and runs this interpreter from the oracle tape; this allows $\bar i$ to describe only a memory address instead of directly coding the interpreter. Consequently, fixing $U$ is not a strong commitment to this particular computer; if we have access to additional computers, we can encode them as interpreters within $q$ for easy access.

For a given program $p \in \bits^\infty$,
define the semimeasure \emph{programmed} by $p$ as
\begin{equation}
\label{eq:mup}
\mu_p(x)
\,:=\, \prob_{\alpha\sim\lambda}
\left(x\sqsubseteq U^\alpha(p)\right).
\end{equation}
$\lambda$ denotes the Lebesgue measure on $\bits^\infty$, meaning the bits of $\alpha\in\bits^\infty$ are independent fair coin flips.
$\mu_p(x)$ gives the fraction of random oracles $\alpha$ for which the program $p$'s output starts with $x$.
Using \eqref{eq:universalturing}, the program $\bar i \mathbf 0$
corresponds to the semimeasure $\mu_{\bar i \mathbf 0}(x) = \prob_{\alpha\sim\lambda}(x \sqsubseteq T_i^\alpha(\mathbf 0))$. Hence, $\{\mu_{\bar i \mathbf 0} : i\in\nats\}$ is the set of all
lower-semicomputable semimeasures \citep[Theorem 4.5.2]{li2019introduction}, which properly includes the set of all computationally feasible predictors.

We also define the Solomonoff-Levin semimeasure relative to a particular oracle $q \in \bits^\infty$ by
\begin{equation}
\label{eq:Mq}
M^q(x)
\,:=\, \prob_{\alpha\sim\lambda}
\left(x\sqsubseteq U^q(\alpha)\right),
\end{equation}
the fraction of random programs $\alpha$ whose output starts with $x$ when given access to $q$.
$M^q$ has a positive probability of sampling from every machine $T_i$, and can therefore be viewed as a universal mixture \citep[Theorem 3.8.8]{hutter2024introduction}. Concretely, since $\prob_{\alpha\sim\lambda}(\bar i\sqsubseteq\alpha)=2^{-|\bar i|}$, $M^q$ equals the mixture \labelcref{eq:bayesmixture} with $w_i := 2^{-|\bar i|}$ and $\nu_i(x) := \prob_{\alpha\sim\lambda}(x\sqsubseteq T_i^q(\alpha))$. There exists a fixed machine $T_j$ (independent of $p,q$) satisfying $T_j^q(p) = U^p(q)$, so that in particular $\nu_j = \mu_q$. Hence, using \labelcref{eq:universalturing}, \labelcref{eq:mup}, and \labelcref{eq:Mq}, $\mu_{\bar j q} = M^q \ge w_j\nu_j = 2^{-|\bar j|}\mu_q$. Since $j$ is fixed,
we can write more compactly %
\begin{equation}
\label{eq:Mvsmu}
M^q(x) \,\geqx\, \mu_q(x),
\end{equation}
where the relational operators $\leqx,\geqx$ hide a multiplicative factor that depends only on $U$ (and $\eqx$ means that both $\leqx$ and $\geqx$ hold). Similarly, we use $\leqp,\geqp,\eqp$ to hide an additive term that depends only on $U$. Since the hidden constants are universal, independent of specific learning algorithms or data, $U$ can
be chosen in such a way as to make all the constants in this paper small \citep[Section 3.9]{li2019introduction}.

We can think of the information vantage point $U^q$ as a customized computer, obtained by equipping the fixed reference computer $U$ with observer-dependent data $q$. Classical Solomonoff induction predicts according to the Bayesian prior $M^{\mathbf 0}$, which may suffer high
regret against a specialized expert $\mu_p$ that takes background information into account.  By placing the observer's information in $q$, \Cref{thm:oraclesolregret} will show that \emph{oracle Solomonoff induction} using $M^q$ has low regret against any expert that is accessible to us in practice. In particular, we need not worry about using the ``wrong'' universal computer: if we have access to a more suitable computer, it would be encoded within $q$.

\subsection{Sample-optimal online learning with oracle Solomonoff induction}

Shannon information theory studies the entropy and mutual information of random variables. It is readily applied to settings that sample repeatedly from i.i.d.\ or stationary ergodic sources, for which aggregate codelengths concentrate near their expectations \citep[Chapter 3 and 16.8]{thomas2006elements}.
Algorithmic information theory defines analogous quantities for individual samples without associated distributions \citep[Section 8.1]{li2019introduction}. Some argue this makes algorithmic information theory relevant to more general (e.g.\ non-ergodic) settings \parencites{kolmogorov1983combinatorial}[Section IV.B]{ebtekar2025foundations}. For our analysis, a useful analogue to the Shannon conditional entropy $H(X\mid Q)$ (for random variables $Q,X$) is the \emph{a priori complexity} of $x\in\bits^*\cup\bits^\infty$ relative to $q\in\bits^\infty$, given by
\begin{equation}
\label{eq:KM}
\KM^q(x) \,:=\, \log \frac{1}{M^q(x)}.
\end{equation}
The logarithm's base amounts to a choice of units \citep{frank2005indefinite}. If taken in base 2, $\KM^q$ is expressed in bits, and is typically slightly less than the prefix Kolmogorov complexity \citep[Theorem 3.8.7]{hutter2024introduction}.

In Shannon theory, the conditional mutual information is $I(P;X\mid Q) := H(X\mid Q) - H(X\mid P,Q)$ for random variables $P,Q,X$. Analogously, we define an asymmetric notion of algorithmic mutual information between $p\in\bits^\infty$ and $x\in\bits^*$, relative to $q\in\bits^\infty$, by
\begin{equation}
\label{eq:mutinf}
\MI^q(p : x)
\,:=\, \KM^q(x) - \KM^{p\oplus q}(x)
\,=\, \log\frac{M^{p\oplus q}(x)}{M^q(x)}.
\end{equation}
Recall that $p \oplus q$ is the interleaving of $p$ and $q$, providing oracle access to both strings.
Other definitions of algorithmic mutual information appear in the literature \parencites[Section 3.1]{gacs2021lecture}{vereshchagin2021proofs},
but $\MI^q$ is simple and useful for our purposes. 
\cref{sec:infproperties} establishes some intuitive properties:
\cref{thm:infobounds} shows
that $0 \leqp \MI^q(p : x)\leqp \min(\KM^q(p),\,\KM^q(x))$,
and \cref{thm:noinfoexnihilo} proves that
information cannot be created \emph{ex nihilo} by any small program, even with randomization.

The following bound states that if a predictor $\mu_p$ substantially outperforms $M^q$, its code $p$ must contain additional information about $x$ that is not in $q$. Since information cannot be created, there is no way to invent such a predictor. Thus, $M^q$ beats the sample-efficiency of all realistic predictors.

\begin{theorem}[Regret of oracle Solomonoff induction]\label{thm:oraclesolregret}
For all $p,q\in\bits^\infty$, the regret of a universal predictor $M^q$ against another predictor $\mu_p$, on any data $x\in\bits^*$, is bounded by
\begin{equation*}
L(M^q,\,x) - L(\mu_p,\,x)
\,\leqp\, \MI^q(p : x) - \MI^p(q : x)
\,\leqp\, \MI^q(p : x)
.\end{equation*}
\end{theorem}

\begin{proof}
Using \cref{eq:loss,eq:Mvsmu},
\[
L(M^q,\,x) - L(\mu_p,\,x)
= \log\frac{\mu_p(x)}{M^q(x)}
\leqp \log\frac{M^p(x)}{M^q(x)}
= \log\frac{M^{p\oplus q}(x)}{M^q(x)}
- \log\frac{M^{p\oplus q}(x)}{M^p(x)}
.\]
Using the definition \eqref{eq:mutinf},
the first term is $\MI^q(p : x)$.
Since $M^{p \oplus q}(x) \eqx M^{q \oplus p}(x)$,
the second term is $\eqp -\MI^p(q : x)$, which is $\leqp 0$ by \cref{thm:infobounds}.
\end{proof}

Combining \Cref{thm:oraclesolregret,thm:infobounds} yields the weaker bound $L(M^q,\,x) - L(\mu_p,\,x) \leqp \KM^q(p)$,
which generalizes Corollary 3.8.10 of \citep{hutter2024introduction}.
\Cref{thm:oraclesolregret} goes further by
clarifying what type of information is relevant to the regret.
For example, suppose $p=r\alpha$ consists of source code $r$ for a randomized algorithm, along with a typical sample $\alpha$ from an infinite source of random bits, neither of which is tailored to the data $x$. Then $\KM^q(p)$ is infinite, while $\MI^q(p : x)$ is small.

Although oracle Solomonoff induction is formally Bayesian, \Cref{thm:oraclesolregret} provides a distribution-free frequentist guarantee. This suggests an \emph{interpretation} of probabilities \parencites{gillies2000philosophical}[Part IV]{hajek2016oxford}: no matter what $x$ the world presents, an idealized learner may model its subjective uncertainty using $M^q$ \citep{rathmanner2011philosophical}. Our position does not take a side in the Bayesian-frequentist divide \citep{mayo2018statistical,sprenger2019bayesian}, but may help illuminate why it exists. Given the framing from \Cref{sec:nfl}, where Bayesian inference fixes its inductive bias up front and frequentist inference accumulates it online, an idealized learner with unlimited resources can afford to commit to a universal prior, and is naturally Bayesian. A resource-bounded learner cannot, so its inductive bias must accumulate over time. Any realistic attempt to approximate $M^q$ would find a never-ending sequence of candidate programs; therefore, we can expect practical methods to acquire a mix of Bayesian and frequentist character \citep{sterkenburg2022truth}.

\section{Neural network generalization} \label{sec:nn-gen}

Some of the most famous results in deep learning are ``universal approximation theorems,'' which say that neural networks can approximate arbitrary functions \citep{cybenko1989approximation,hornik1991approximation,deepnarrow,yun2020transformers}. Referring back to \cref{fig:circles}, this is like saying neural networks are a large class; they will fit the training data, but uniform convergence will not apply.
While this is frequently cited as ``the reason'' that neural networks work well in practice,
this is misleading:
equivalent universal approximation results are shared by Gaussian kernel spaces \citep[Section 4.6]{steinwart-christmann}
and even histograms \citep[Section 7.4]{royden-fitzpatrick}.

In order to generalize well, we need more: a universal inductive bias, approaching something like $M^q$ in a limit of large model size and compute. For an untrained model, the information vantage point $U^q$ may be informed by common computing primitives. Or, if we want to understand learning from a pretrained baseline, $q$ may include training data, or its representation as model weights.
There are some indications that modern models may correspond to a simplicity bias in universal computing terms \citep{valleperez2019deep,mingard2025deep}.
Such arguments have recently been made to explain various empirical results in practical networks \citep{deletanglanguage,goldblum2024freelunchtheoremkolmogorov,huh2024platonicrepresentationhypothesis,huang2024compressionrepresentsintelligencelinearly},
as well as in toy models \citep{ren2024understandingsimplicitybiascompositional}.

How does this bias occur?
In addition to mechanisms discussed by the previous papers,
\citet{buzaglo2024uniform} prove that wide feedforward networks have disproportionately large regions of parameter space corresponding to functions also computed by much narrower networks.
Thus, interpolating solutions found via random search are more likely to correspond to a narrow network,
supporting empirical findings that random search generalizes well \citep{chiang2023loss,pakman2026revisitingvolumehypothesis}.
Narrow networks are not quite the same as short programs, but they are closely related, as small feedforward networks can directly encode small binary circuits \citep{circuit-complexity}.
Adding recurrence and an external memory turns circuits into universal Turing machines; transformers with chain of thought move toward that \citep{attn-turing-complete,expressive-cot,young2024transformers,nowak2024representational,pencil,li:constant-size-turing}.
Whether the parameter space bias extends to favouring all short programs,
and the consequences for practical optimization methods, remain to be seen.

\citet{jacot2026resnets} suggests another distinct mechanism:
under certain assumptions, a particular norm of ResNet parameters is closely related to the size of a minimal circuit implementing the same function.
It is not yet clear whether practical optimization algorithms minimize that parameter norm.

Large language models learn to learn: after pre-training, they are able to ``learn in-context'' within a text sequence \citep{dong2024survey}, in a manner that has been argued to resemble Bayesian inference \citep{xie2022explanationincontextlearningimplicit,wakayama2025incontextlearningprovablybayesian}, or even specifically Solomonoff induction \citep{wan2025large}. Without needing any real data at all, \citet{grau2024learning} train a neural network to approximate Solomonoff induction in-context, using synthetic data generated by random programs with a time bound. \citet{cowsik2026selfplay} extend this approach by co-evolving a generator for the random programs. A related line of work by
\citet{shaw2026bridging} trains a neural network to approximate the Minimum Description Length principle \citep{blum2003pac,wallace2005statistical,hutter2009discrete,grunwald2019minimum,zhu25a,li2026prediction}.

According to No Free Lunch, fitting to training inputs, even if sampled from a Solomonoff measure, still allows for arbitrary behaviour on unseen inputs. Therefore, the strong generalization performance that these authors elicit on unseen inputs cannot be solely attributed to their synthetic data. We should investigate the extent to which successful machine learning algorithms and architectures exhibit a pseudo-universal inductive bias, prior to taking any training data.

\section{Discussion}
\citet{goldblum2024freelunchtheoremkolmogorov} recently argued that
algorithmic information theory explains how to beat No Free Lunch, because
``low-complexity structure shared by real-world datasets and machine learning models enables broad generalization across domains and sample sizes with a single model class.''
We do not disagree,
but we ask: low complexity relative to
\emph{what}, and \emph{why} \citep{sterkenburg2026solomonoff}?
Any argument based on empirical success -- as theirs is --
reverts back to the problem of meta-induction as in \cref{thm:meta-nfl}.
Our accessibility framing provides a crisper account of optimal learning:
while we cannot be sure that $M^q$ will predict accurately,
\cref{thm:oraclesolregret} guarantees that no realistic alternative is substantially better.

This framing brings prescriptive value:
generalist AI systems can aim to approximate the regret bound in \cref{thm:oraclesolregret}, at least when their compute budget greatly exceeds that of an expert $\mu_p$.
It is also of descriptive value:
there is reason to think that deep learning and transformers have moved us closer to universal induction. We hope to apply this perspective toward understanding what AI systems do, and what they \emph{should} do in challenging settings, such as in-context out-of-distribution generalization.
Where real AI systems fall short of true universality -- due to computational bounds or other differences -- we expect $M^q$ to serve as a useful baseline from which to understand deviations. Models of agency based on $M^q$, such as AIXI, can help us anticipate future agentic capabilities and behaviours, and identify safety risks before they occur \parencites[Chapter 15]{hutter2024introduction}{meulemans2025embedded}{ebtekar2026golden}.

Furthermore, we saw that the study of universal induction inspires additional concepts such as the algorithmic mutual information, which quantifies a model's knowledge of a dataset. Algorithmic information theory is a mathematical discipline replete with concepts relevant to machine learning, accounting for factors such as accessibility, model size, time bounds, and even novelty \citep{filan2016loss,vereshchagin2016algorithmic,li2019introduction,zaffora2025bayesian,ebtekar2025toward,finzi2026entropy,ebtekar2026golden,voitovych2026learning}. We believe its applications are severely underexplored, and its relevance grows rapidly with advancing AI capabilities.

\begin{ack}

The authors would like to thank
Cole Wyeth for particularly extensive feedback, as well as Tom Sterkenburg, Micah Goldblum, Justice Sefas, David Quarel, David Wolpert, David Dowe, Tor Lattimore, Daniel Herrmann, Aydin Mohseni, Francesca Zaffora Blando, Gerhard Schurz, Nathan Srebro, Henrik Marklund, Alex Infanger, Lily Stelling,  Annabelle Rondeau, and David Zheng for productive conversations.

This work was supported by the Canada CIFAR AI Chairs program, Coefficient Giving, and the AI Security Institute (AISI) through the AI Safety Tactical Opportunities Fund (AISTOF).
\end{ack}

\printbibliography

\clearpage
\appendix
\crefalias{section}{appendix}
\crefalias{subsection}{appendix}

\section{Properties of algorithmic mutual information} \label{sec:infproperties}
Throughout this section, it will be convenient to take $\KM$ and $\MI$ to be measured in bits (i.e.\ using base 2 logarithms in \labelcref{eq:KM,eq:mutinf}), which are also the unit of length for bit strings.

\subsection{Supporting derivations}

Our first result is analogous to the Shannon theory inequality for random variables $P,Q,X$:
\[0\le I(P;X\mid Q)\le\min\left(H(P\mid Q),\,H(X\mid Q)\right).\]
\begin{proposition}[Information bounds]
\label{thm:infobounds}
For all $p,q\in\bits^\infty$ and $x\in\bits^*$,
\begin{equation*}
0\le \KM^q(x) \leqp |x|,\qquad
0\leqp \MI^q(p : x)\leqp \min(\KM^q(p),\,\KM^q(x)).
\end{equation*}
\end{proposition}

\begin{proof}
By considering the machine $T_i^q(\alpha):=\alpha$ that copies its program to the output tape, we obtain
\begin{align*}
M^q(x)
&= \prob_{\gamma\sim\lambda}\left(x\sqsubseteq U^q(\gamma)\right)
\\&\ge 2^{-|\bar i|-|x|}\prob_{\alpha\sim\lambda}\left(x\sqsubseteq U^q(\bar i x \alpha)\right)
\\&= 2^{-|\bar i|-|x|}\prob_{\alpha\sim\lambda}\left(x\sqsubseteq T_i^q(x \alpha)\right)
\\&= 2^{-|\bar i|-|x|}\prob_{\alpha\sim\lambda}\left(x\sqsubseteq x \alpha\right)
\\&= 2^{-|\bar i|-|x|}.
\end{align*}
Taking logarithms and applying \labelcref{eq:KM} yields $0\le \KM^q(x)\le |\bar i| + |x| \leqp |x|$.

By considering the machine $T_k^{p\oplus q}(\alpha):= U^q(\alpha)$ that ignores $p$, we obtain $M^{p\oplus q}(x)\geq 2^{-|\bar k|} M^q(x)\eqx M^q(x)$. Hence from the definition \labelcref{eq:mutinf}, $0\leqp \MI^q(p : x)\le \KM^q(x)$.

For the remaining inequality, consider the machine $T_j^q(\alpha\oplus\beta):= U^{U^q(\beta)\oplus q}(\alpha)$. This involves an abuse of notation, as $U^q(\beta)$ is not necessarily an infinite string; if the machine tries to access an invalid index, we take it to cease writing to the output tape. Since $j$ is a universal constant,
\begin{align*}
M^q(x)
&= \prob_{\gamma\sim\lambda}\left(x\sqsubseteq U^q(\gamma)\right)
\\&\geqx \prob_{\alpha,\beta\sim\lambda}\left(x\sqsubseteq U^q(\bar j (\alpha\oplus \beta))\right)
\\&= \prob_{\alpha,\beta\sim\lambda}\left(x\sqsubseteq T_j^q(\alpha\oplus \beta)\right)
\\&= \prob_{\alpha,\beta\sim\lambda}\left(x\sqsubseteq U^{U^q(\beta)\oplus q}(\alpha)\right)
\\&\ge \prob_{\beta\sim\lambda}(U^q(\beta)=p) \cdot \prob_{\alpha\sim\lambda}\left(x\sqsubseteq U^{p\oplus q}(\alpha)\right)
\\&= M^q(p)\cdot M^{p\oplus q}(x).
\end{align*}
By rearranging and taking logarithms, $\MI^q(p : x)\leqp \KM^q(p)$.
\end{proof}

Next, we show that no deterministic or randomized meta-algorithm $r$ can reliably create $p$ that has high mutual information with $x$. \citet{levin1984randomness,vereshchagin2021proofs} proved similar results for symmetric notions of mutual information, under the name \emph{conservation of randomness}; \citet{ebtekar2025foundations} frame one such result as a generalization of the second law of thermodynamics.

\begin{theorem}[No information ex nihilo]
\label{thm:noinfoexnihilo}
Suppose $q\in\bits^\infty$, and $r\in\bits^*$ be such that the randomly generated string $p := U^q(r\beta)$ is infinitely long for almost all $\beta\sim\lambda$. Then for all $x\in\bits^*$,
\begin{equation*}
\expect_p \left[\MI^q(p:x)\right]
\le \log_2 \expect_p \left[2^{\MI^q(p:x)}\right]
\leqp \min\bigl( |r|,\,\MI^q(\bar r\mathbf 0:x) \bigr),
\end{equation*}
where we fix a standard self-delimiting encoding $\bar r:=\overline{|r|+1}r$ of strings.
\end{theorem}

\begin{proof}
From Jensen's inequality and the definition \labelcref{eq:mutinf},
we have
\[
{\expect_p \left[\MI^q(p:x)\right]}
\le \log_2 \expect_p \left[2^{\MI^q(p:x)}\right]
= \log_2 \frac{\expect_p M^{p\oplus q}(x)}{M^q(x)}
.\]

To compare the latter expression with $|r|$,
consider again the machine $T_j^q(\alpha\oplus\beta) := U^{U^q(\beta)\oplus q}(\alpha)$. A random string $\gamma\sim\lambda$ has probability $2^{-|\bar j|-|r|}$ of having the form $\bar j (\alpha\oplus r\beta)$ for some $\alpha,\beta\in\bits^\infty$. Since $j$ is constant,
\begin{align*}
2^{|r|} M^q(x)
&= 2^{|r|} \prob_{\gamma\sim\lambda}\left(x\sqsubseteq U^q(\gamma)\right)
\\&\geqx \prob_{\alpha,\beta\sim\lambda}\left(x\sqsubseteq U^q(\bar j(\alpha\oplus r\beta))\right)
\\&= \expect_{\beta\sim\lambda}\prob_{\alpha\sim\lambda}\left(x\sqsubseteq T_j^q(\alpha\oplus r\beta)\right)
\\&= \expect_{\beta\sim\lambda}\prob_{\alpha\sim\lambda}\left(x\sqsubseteq U^{U^q(r\beta)\oplus q}(\alpha)\right)
\\&= \expect_p \prob_{\alpha\sim\lambda}\left(x\sqsubseteq U^{p\oplus q}(\alpha)\right)
\\&= \expect_p M^{p\oplus q}(x)
;\end{align*}
hence,
\[
\log_2 \frac{\expect_p M^{p\oplus q}(x) }{ M^q(x) }
\leqp \log_2 \frac{2^{|r|}M^q(x) }{ M^q(x) }
= |r|
.\]

For the remaining comparison,
consider the machine $T_k^{\bar rs \oplus q}(\alpha\oplus\beta) := U^{U^q(r\beta)\oplus q}(\alpha)$. A random string $\gamma\sim\lambda$ has probability $2^{-|\bar k|}$ of having the form $\bar k (\alpha\oplus\beta)$.
Since $k$ is constant,
\begin{align*}
M^{\bar r\mathbf 0\oplus q}(x)
&= \prob_{\gamma\sim\lambda}\left(x\sqsubseteq U^{\bar r\mathbf 0\oplus q}(\gamma)\right)
\\&\geqx \prob_{\alpha,\beta\sim\lambda}\left(x\sqsubseteq U^{\bar r\mathbf 0\oplus q}(\bar k(\alpha\oplus\beta))\right)
\\&= \expect_{\beta\sim\lambda}\prob_{\alpha\sim\lambda}\left(x\sqsubseteq T_k^{\bar r\mathbf 0\oplus q}(\alpha\oplus\beta)\right)
\\&= \expect_{\beta\sim\lambda}\prob_{\alpha\sim\lambda}\left(x\sqsubseteq U^{U^q(r\beta)\oplus q}(\alpha)\right)
\\&= \expect_p\prob_{\alpha\sim\lambda}\left(x\sqsubseteq U^{p\oplus q}(\alpha)\right)
\\&= \expect_p M^{p\oplus q}(x);
\end{align*}
hence,
\[
\log_2 \frac{\expect_p M^{p\oplus q}(x)}{M^q(x)}
\leqp \log_2 \frac{M^{\bar r\mathbf 0\oplus q}(x)}{M^q(x)}
= \MI^q(\bar r \mathbf 0 : x)
.\qedhere\]
\end{proof}

Suppose we want to design a predictor $\mu_p$ that substantially outperforms $M^q$. According to \cref{thm:oraclesolregret}, this requires $\MI^q(p:x)$ to be high. What kind of process can generate a suitable program $p$? If we accept the physical Church-Turing thesis, the design process must follow a preexisting meta-algorithm, perhaps inside our brain, making it accessible from our information vantage point. We can model any accessible method as using very little new code $r$, which uses the information vantage point $q$ and random source $\beta$ to generate $p := U^q(r\beta)$. \Cref{thm:noinfoexnihilo} implies that with high probability, such a process produces low $\MI^q(p:x)$, so $\mu_p$ will not predict substantially better than $M^q$.

For the purposes of the following subsection, we need one additional result.

\begin{theorem}[Information supplied by a self-delimiting prefix]
\label{thm:prefixinformation}
Suppose $p,q\in\bits^\infty$ and $r,x\in\bits^*$, with $r$ belonging to a fixed computably enumerable set, no member of which is a proper prefix of another. Then,
\begin{equation*}
\MI^q(rp:x)
\leqp |r| + \MI^q(p:x).
\end{equation*}
\end{theorem}

\begin{proof}
Since $r$ is self-delimiting, there exists a fixed machine $T_j$ that reads $r$, and simulates $U$ on a modified oracle and program, so that $T_j^{p\oplus q}(r\alpha) := U^{(rp)\oplus q}(\alpha)$. A random string $\gamma\sim\lambda$ has probability $2^{-|\bar j|-|r|}$ of starting with $\bar j r$. Since $j$ is constant,
\begin{align*}
M^{p\oplus q}(x)
&= \prob_{\gamma\sim\lambda}
   \left(x\sqsubseteq U^{p\oplus q}(\gamma)\right)
\\&\geqx 2^{-|r|}
   \prob_{\alpha\sim\lambda}
   \left(x\sqsubseteq U^{p\oplus q}(\bar j r\alpha)\right)
\\&= 2^{-|r|}
   \prob_{\alpha\sim\lambda}
   \left(x\sqsubseteq T_j^{p\oplus q}(r\alpha)\right)
\\&= 2^{-|r|}
   \prob_{\alpha\sim\lambda}
   \left(x\sqsubseteq U^{(rp)\oplus q}(\alpha)\right)
\\&= 2^{-|r|} M^{(rp)\oplus q}(x).
\end{align*}
Applying the definition \labelcref{eq:mutinf},
\[
\MI^q(rp:x) - \MI^q(p:x)
= \log_2\frac{M^{(rp)\oplus q}(x)}{M^{p\oplus q}(x)}
\leqp |r|.
\qedhere \]
\end{proof}

\subsection{Anytime prediction}

In \labelcref{eq:mup}, $p$ can be viewed as the code for an anytime algorithm that makes predictions according to successive approximations of $\mu_p$. Indeed, we can simulate $U^\alpha(p)$ for all possible $\alpha\in\bits^\infty$ in parallel: each time $U$ reads a new bit of $\alpha$, fork the computation into cases where the bit is $\texttt{0}$ or $\texttt{1}$, allocating an equal share of time to each. Each time $U$ outputs a bit, add $2^{-n}$ to our estimate of $\mu_p(x)$, where $n$ is the number of bits read from $\alpha$ so far, and $x$ is the output so far. Since these estimates are monotonically increasing and converge to $\mu_p(x)$, we say the prior $\mu_p$ is \emph{lower-semicomputable} (if $p$ is a finite program padded with $\mathbf 0$; otherwise, it is lower-semicomputable \emph{relative} to $p$).

To complete the anytime prediction algorithm, we can prepend to $p$ a constant instruction along with an encoding of the history $h\in\bits^*$, to make it output successive estimates of the posterior
\[\mu_p(x\mid h)
:= \frac{\mu_p(hx)}{\mu_p(h)}.\] 
A ratio of increasing estimates is not itself necessarily increasing, so the posterior is not in general lower-semicomputable.
It is, however, \emph{limit-computable} \parencites[Section 3]{leike2015computability}[Section 3.3]{sterkenburg2026solomonoff}, since
the ratio of estimates converges (non-monotonically) to the posterior $\mu_p(x\mid h)$.

Since our posterior predictions are only limit-computable,
one might wonder about the performance of the broader class of anytime algorithms with this property.
That is, instead of the lower-semicomputable prior $\mu_p$, suppose we start with a general limit-computable prior; the posterior, being a ratio of priors, would then still be limit-computable.

One such algorithm may estimate the posteriors of $M^q$ with sufficient precision to predict Putnam's adversarial $x$ (of any desired length), which $M^q$ fails to predict, obtaining $M^q(x)\le 2^{-|x|}$. While $x$ (and hence, the point mass prior concentrated on it) is limit-computable (relative to $q$), it costs an extreme amount of time, even compared to some useful approximations of $M^q$. For a concrete comparison, note that we can approximate $M^q$ by simulating all possible programs $p$ in parallel as above, allocating a fraction $2^{-n}$ of our computation time to each program prefix of length $n$. This approximation is competitive against every fast and short expert, with a runtime overhead ``only'' exponential in the expert's length. In contrast, to approximate $M^q$ well enough to produce an adversarial $x$, we must wait for even the slowest programs of the desired length to halt. The runtime is much worse than exponential, growing faster than any computable function.

In general, given an anytime algorithm $p$ and a time limit $t$,
let $\nu_{p,t}$ denote the prior semimeasure that results from running $p$ for $t$ time steps.
Since $\nu_{p,t}$ is computable, there exists some $r := r(t)$
such that $\mu_{rp} = \nu_{p,t}$:
$r$ encodes the time limit,
along with some wrapper code to run the subsequent algorithm $p$ until the specified time $t$, and sample according to its latest estimates.
Applying \cref{thm:oraclesolregret,thm:prefixinformation} to $\mu_{rp}$ yields the regret bound
\begin{equation}
\label{eq:anytimebound}
L(M^q,\,x) - L(\nu_{p,t},\,x) \leqp \MI^q(rp : x)
\leqp |r| + \MI^q(p : x).
\end{equation}

By setting $t$ to the $k$th prefix-free busy beaver number \citep{bienvenu2012random}, we can get $|r|\eqp k$.
As the busy beaver numbers grow faster than any computable function,
an anytime algorithm can only substantially increase
the right-hand side of \labelcref{eq:anytimebound} after an unfathomable amount of time.\footnote{%
We are interested in times $t$ at which $\nu_{p,t}$, or at least $L(\nu_{p,t},\,x)$, has permanently settled to good enough approximations of its limit as $t\rightarrow\infty$; hence, no generality is lost by increasing $t$ to the next prefix-free busy beaver number. These numbers are defined relative to a particular Turing machine,
but another self-contained way to
specify a similarly large time limit
would be to
describe a particular Turing machine that runs
for a large number of steps before halting.
(The program $rp$ would alternate between simulating steps of this machine and steps of $p$, until the machine halts.)
The longest-running 5-state machine takes 47,176,870 steps before halting \citep{bb5}.
With 6 states, a machine is known that runs for over $2 \uparrow \uparrow \uparrow 5$ steps \citep{bb6-lb}, the kind of number for which ``the number of atoms in the Universe times the number of femtoseconds since the Big Bang'' is a rounding error.
An arbitrary $k$-state Turing machine
takes $O(k \log k)$ bits to specify,
giving $|r|=O(k \log k)$.
}
This increase is more than a proof artifact: the sheer scale of busy beaver numbers makes them informative about some otherwise undecidable problems \citep{chaitin1987computing}.

\section{Additional considerations regarding oracle Solomonoff induction} \label{sec:osiproperties}
\subsection{The reference universal computer}
Fundamentally, the information vantage point consists of not only the oracle $q$, but the entire relativized machine $U^q$. In the case of a finite prefix $q_{<n}$ padded with $\mathbf 0:=\texttt{0}^\infty$, the oracle $q:=q_{<n}\mathbf 0$ can be hardcoded into the description of another universal computer $V$, such that $V^{\mathbf 0}=U^q$. Our ``oracle Solomonoff induction'' is then just classical Solomonoff induction with respect to $V$.

Nonetheless, it is useful to separate out a ``minimal'' $U$ to treat as a universal reference. The intuitive idea is as follows: while each person on Earth likely has a different vantage point, they are also likely to share some common knowledge. Provided that everyone can agree on a ``Schelling point'' machine $U$, individuals retain the freedom to write programs in their preferred languages, which translate onto $U$ via interpreters included in their personal instantiation of $q$. $U$ can be chosen to be accessible from physics \citep{deutsch2013computation,ozkural2015ultimate,janzing2018does,kolchinsky2020thermodynamic,wolpert2024implications,ebtekar2025foundations}, and tailored to minimize the hidden constants in our regret bounds \citep[Section 3.9]{li2019introduction}. By agreeing on $U$, we ensure that the hidden constants in our results have small universal bounds independent of $q,p,x$. The oracle $q$ is itself not a universal constant, so we explicitly account for its influence in our bounds.

What if humanity's common knowledge includes multiple machines, say $U$, $V$, and $W$? Given $U$ satisfying \labelcref{eq:universalturing}, there exists a fixed machine $T_r$, such that $U^{\bar n_1\ldots\bar n_k q}(\bar r\bar ip)=T_r^{\bar n_1\ldots\bar n_k q}(\bar ip):=T_{n_i}^{\bar n_1\ldots\bar n_k q}(p)$ for all $i,k,n_1,\ldots,n_k\in\nats$ and $p,q\in\bits^\infty$ with $i\le k$. Now set $\bar n_1,\bar n_2,\bar n_3$ to interpreters of $U,V,W$ in the language of $U$, meaning that $T_{n_1}=U,\,T_{n_2}=V,\,T_{n_3}=W$. Then, $M^{\bar n_1\bar n_2\bar n_3q}$, which is based on $U$, dominates the analogous semimeasures for the machines $V,W$, up to the constant factor $2^{-|\bar r|-|\bar i|}$.

$|\bar i|$ is the cost of a pointer to the relevant machine, as they cannot all share a canonical address. Meanwhile, $|\bar r|$ is another universal constant to optimize; indeed, we can select $U$ such that the distinguished machines used in our proofs (e.g.\ $T_r$) all have short interpreters. As long as we do this, the specific choice of $U$ matters little.

\subsection{Fresh restarting \texorpdfstring{$M^{h\mathbf 0}(x)$}{} vs lifelong learning \texorpdfstring{$M^{\mathbf 0}(hx)$}{}}
Suppose we want to predict $x\in\bits^*$, using some prior history $h\in\bits^*$. We can take either (1) a \emph{fresh restarting} approach, directly predicting $x$ according to the semimeasure $M^{h\mathbf 0}(x)$ that accesses $h$ via the oracle; or (2) a \emph{lifelong learning} approach, predicting the concatenation $hx$ according to the semimeasure $M^{\mathbf 0}(hx)$, so that the predictor sees $h$ before $x$.

According to \Cref{thm:oraclesolregret}, these two approaches optimize different objectives against different classes of experts. While both can access $h$ while predicting $x$, a lifelong learner with sufficient experience must converge to the conclusion supported by $h$ \citep{solomonoff1978complexity}, whereas a freshly restarted learner must give substantial weight to other alternatives.

For example, suppose $h:=\texttt{0}^n$ for some large $n$, while the first bit of $x$ is $x_1:=\texttt{1}$. In order to achieve bounded regret against a straightforward ``always \texttt{0}'' expert, our predictions must converge as $h$ gets longer; that is, $\lim_{n\rightarrow\infty}M^{\mathbf 0}(\texttt{1}\mid \texttt{0}^n) = 0$. There is no need to hedge substantially against the possibility of seeing $x_1=\texttt{1}$, because we only compete against accessible experts; without knowing $n$, no expert that predicts the long string of \texttt{0}s will know to change to $\texttt{1}$ after exactly $n$ steps. In contrast, the fresh restarting learner must compete with experts that are \emph{only} evaluated on $x$, and must therefore hedge against both possibilities for $x_1$! Formally, $\inf_{h\in\bits^*} M^{h\mathbf 0}(\texttt{1}) > 0$.

More generally, given a long sequence consisting of multiple episodes, predicting each new episode $e^{(i)}\in\bits^*$ with a freshly instantiated universal semimeasure $M^{e^{(1)}\ldots e^{(i-1)}\mathbf 0}(e^{(i)})$ enables \Cref{thm:oraclesolregret} to bound the regret separately per episode. The price is that the bound on \emph{total} regret is now multiplied by the number of episodes. If we want to minimize the total loss across all episodes, then the lifelong learning Solomonoff induction $M^{\mathbf 0}(e^{(1)}\ldots e^{(N)})=\prod_{i=1}^N M^{\mathbf 0}(e^{(i)}\mid e^{(1)}\ldots e^{(i-1)})$ is superior.

Prior work on Solomonoff induction generally focuses on the lifelong learning version, representing past episodes as history rather than as a vantage point oracle. Its performance on future episodes can still be bounded, albeit with additional error terms \parencites{chernov2007algorithmic}[Section 8.3]{rathmanner2011philosophical}.

\subsection{Unknowability and safety}

Much of AI research seeks to improve learning in challenging settings involving out-of-distribution generalization. Since $M^q$ gives an upper bound on the performance of all accessible predictors, estimating the performance of $M^q$ indicates how much room there may be for improvement.

For example, image classifiers often use spurious features, e.g.\ distinguishing a water bird from a land bird by the presence of water in the background \citep{group-dro}.
Is this behaviour a failure of our learning algorithms,
or is it genuinely the best inference supported by our training data?
If $M^q$ successfully avoids mistakes that our classifier makes on out-of-distribution data,
then general-purpose algorithmic improvements may be possible.
On the other hand, if $M^q$ replicates the undesired behaviour, then the flaw is not in the algorithm.
Instead, we should carefully examine what assumptions or additional data lead us to intuit better inferences than those made by $M^q$,
so that we can make the relevant information available to our learning algorithms.

Another application of reasoning about what
$M^q$ does not know is found in AI safety.
\citet{cohen2020pessimism,ebtekar2026golden}
propose techniques to help prevent superintelligent agents
from autonomously taking dangerous, novel courses of action,
based on the inherent uncertainty that even maximally-intelligent Solomonoff induction must maintain about its environment.

No Free Lunch exposes a fundamental ambiguity in generalization. To see this, consider the following variant of the ``grue'' problem \citep{goodman1983fact} in AI safety. We might expect an AI system trained on safe responses to be safe at deployment. However, consider any predicate $E$ that holds throughout training, but ceases to hold at deployment, and define the word ``saferous'' to mean ``safe while $E$, dangerous otherwise.'' The same training procedure can equally well be interpreted as training the AI on ``saferous'' responses, yielding dangerous responses at deployment!
Most language model jailbreak techniques can be seen as introducing such a predicate $E$ \citep{wei2023jailbroken}.
In future work, we hope to apply algorithmic information theory toward a more rigorous science of AI safety.

\subsection{Optimality in computable environments}

In \Cref{thm:oraclesolregret}, the universal predictor's regret bound holds for every observation sequence $x$, but depends on how much information the rival predictor $\mu_p$ has about $x$. There is a dual bound, which does not depend on the rival predictor, but instead depends on the complexity of the observation-generating process, henceforth called the \emph{environment}. If a computable probability measure $\mu_p$ generates $x$, then $\mu_p$ is also the Bayes-optimal predictor. Consequently, any mean regret bound that holds against $\mu_p$ will in fact apply against all rivals, computable or otherwise.

In this dual setting, observations are random variables, so optimality is measured with respect to the mean cross-entropy loss that a predictor $\nu$ accumulates in the first $t\in\nats$ steps:
\[L_t(\nu,\,\mu)
:= \E_{x\sim\mu} L(\nu,\,x_{\le t})
= -\E_{x\sim\mu}\log\nu(x_{\le t}).\]

\begin{corollary}[Optimality in computable environments]\label{cor:oraclesolcomputable}
For all $p,q\in\bits^\infty$ such that $\mu_p$ is a probability measure (i.e.\ not merely a semimeasure), and $t\in\nats$, we have
\begin{equation*}
L_t(M^q,\,\mu_p)-\inf_\nu L_t(\nu,\,\mu_p)
\leqp \KM^q(p)
.\end{equation*}
\end{corollary}

\begin{proof}

\Cref{thm:oraclesolregret,thm:infobounds} together imply
\[L(M^q,\,x)-L(\mu_p,\,x)
\leqp \KM^q(p).\]
Taking $x$ to be the first $t$ bits generated by $\mu_p$, and taking expectations,
\[L_t(M^q,\,\mu_p)-L_t(\mu_p,\,\mu_p)
\leqp \KM^q(p).\]

Hence, for every semimeasure $\nu$ (computable or otherwise),
\begin{align*}
L_t(M^q,\,\mu_p)-L_t(\nu,\,\mu_p)
&\leqp \KM^q(p) + L_t(\mu_p,\,\mu_p)-L_t(\nu,\,\mu_p)
\\&= \KM^q(p) - \E_{x\sim\mu_p}\left[\log\frac{\mu_p(x_{\le t})}{\nu(x_{\le t})}\right]
\\&= \KM^q(p) - \kl{\mu_p(\cdot_{\le t})}{\nu(\cdot_{\le t})}
\\&\le \KM^q(p).
\end{align*}
In the last line, we used the fact that the Kullback-Leibler divergence is non-negative even when its second argument is a semimeasure; normalizing would only decrease its value.
Taking the supremum of both sides over all semimeasures $\nu$ completes the proof.
\end{proof}

We can compare the pair of optimality results for oracle Solomonoff induction: \Cref{thm:oraclesolregret} applies to all possible observations but has a slack equal to the information content of the rival, whereas \Cref{cor:oraclesolcomputable} applies against all rivals but has a slack equal to the relativized complexity of the stochastic process generating the observations. 

The latter framing is more common in learning theory, where one often starts with probabilistic assumptions and then seeks an optimal learner.
For example, a statistical mechanics perspective on learning assumes that the data is sampled from a ``teacher'' model, which is itself sampled from some prior, and then the student's prediction loss is averaged over the sampling of both the teacher and the data \citep{Seung1992Apr,Watkin1993Apr,Zdeborova2016Sep,Aubin2018}.
The computable world model of \Cref{cor:oraclesolcomputable} includes most such settings:
if the data $x\sim\mu_p$ consists of concatenated self-delimiting tokens sampled i.i.d.\ from a teacher model,
then $\KM^q(p)$ is approximately the teacher's complexity,
which is small for a wide variety of statistical models used in practice.

Despite its intuitive appeal, when it comes to general online learning, we find the framing of \Cref{cor:oraclesolcomputable} to be weak in both philosophical and practical terms.
The information vantage point cannot constrain the environment (as some of it is inaccessible until observed), so the relativized complexity $\KM^q(p)$ may well be high. This weakness is well-known, and historically contributed to skepticism regarding the efficacy of Solomonoff induction \citep{sterkenburg2016solomonoff}. In contrast, the information vantage point \emph{does} constrain which rival predictors are accessible to us; therefore, combining the physical Church-Turing thesis with \Cref{thm:noinfoexnihilo} results in a very small slack $\MI^q(p : x)$ for \Cref{thm:oraclesolregret}.
This amounts to a reformulation of our answer to meta-NFL (\cref{thm:meta-nfl}): we cannot be sure that the world adheres to any structure, but we can understand our own predictive capabilities.

Another important difference is that \Cref{cor:oraclesolcomputable} assigns ``objective'' probabilities to the observations $x$, effectively splitting their complexity into a ``model part'' given by the description of $\mu_p$, and a ``noise part'' as $x$ is sampled from $\mu_p$. While this seems natural in a lot of settings, it leaves some ambiguity: for example, if many generating distributions are initially plausible, should we take $\mu_p$ to be (1) their Bayesian mixture, (2) the specific distribution that best fits observations, or (3) a point mass of probability one on the actual observed sequence? To avoid such ambiguities in the problem specification, we prefer to leave the task of probabilistic modelling up to the predictor. Indeed, the only probabilities referenced in \Cref{thm:oraclesolregret} are the predictions. Instead of being imposed externally, these probabilities emerge from a predictor's inference calculations on the actual observation $x$, offering a universal subjective interpretation of probability \parencites{gillies2000philosophical}{rathmanner2011philosophical}[Part IV]{hajek2016oxford}. In this view, $x$ is considered random to the extent that it is unpredictable by accessible computations \citep{MartinLof1966}.

\subsection{Other loss functions}
An important limitation of \Cref{thm:oraclesolregret} is that it only applies to the log loss. The \emph{strong convexity} of the log loss enables a predictor to hedge when it is uncertain \citep{vovk2001competitive}. In contrast, many real settings require commitment to a specific action.

An important case is the 0-1 loss. Instead of probabilistic predictions, it demands a series of discrete predictions that the next bit $x_t$ will be either \texttt{0} or \texttt{1}, and the loss is simply the number of mistakes. Given a Bayesian prior $\nu$, it seems reasonable to choose the prediction that minimizes expected loss. For the 0-1 loss, this is the maximum a posteriori (MAP) estimate, so that
\begin{equation*}
L_\text{0-1}(\nu,\,x) :=
\sum_{i=1}^{|x|} \indic(\nu(x_{<i}x_i) \le \nu(x_{<i}\neg x_i)).
\end{equation*}

Unfortunately, the adversarial example of \citet{putnam1963degree} (also see \cref{sec:bayesian}) returns in full force to produce a sequence $x$ on which this strategy -- and in fact, \emph{any} strategy -- predicts every bit incorrectly. If the prior is $M^q$, no computable expert can predict this $x$; nonetheless, an always-\texttt{0} or always-\texttt{1} ``expert'' would guess correctly at least half of the time, resulting in a large regret. Since these experts are simple, \Cref{thm:infobounds} implies $\MI^q(p:x)\eqp 0$, so the bound in \Cref{thm:oraclesolregret} does not hold.

On the other hand, this adversarial $x$ does not seem like a ``natural'' sequence that could arise, for example, by sampling from a computable distribution. Its construction depends on the incomputable semimeasure $M^q$, which is related to the \emph{Turing jump} $q'$ of $q$ \citep{nies2009computability,downey2010algorithmic}. We speculate that it may be possible to bound the regret for general loss functions in terms of $\MI^q(p\oplus q':x)$, leaving the precise expression and analysis to future work.

\section{No Free Lunch and complexity measures} \label{sec:nfl-slt}

No Free Lunch means that any generalization bound we prove must depend on an inductive bias. How can we understand the inductive bias implicit in some classical statistical learning results?

\subsection{Uniform convergence}
We will first review some relevant textbook results.
More details are available from many sources,
such as the recent book by
\citet[Section 4.5]{ltfp}.

Consider a learning problem
where we observe independent data $Z = (z_1, \dots, z_m) \sim \dist^m$
-- i.e.\ the $z_i$ are i.i.d. from $\dist$ --
and wish to identify a hypothesis $h \in \cH$
having small risk $L_\dist(h) = \E_{z \sim \dist} \ell(h, z)$,
often based on the empirical risk
$L_Z(h) = \frac{1}{m} \sum_{i=1}^m \ell(h, z_i)$.
In typical prediction settings, we have
$z = (x, y)$
and $\ell(h, (x, y))$
might be $\indic(h(x) \ne y)$,
$\norm{h(x) - y}^2$,
or $- \log( h(x) \cdot y)$.

A primary tool of statistical learning theory is
\emph{uniform convergence},
based on
the trivial inequality
\begin{equation} \label{eq:uniform-conv}
    L_\dist(\hat h)
    \le L_Z(\hat h) + \sup_{h \in \cH}[ L_\dist(h) - L_Z(h) ]
.\end{equation}
If we have a reasonable upper bound on
$\sup_{h \in \cH}[ L_\dist(h) - L_Z(h) ]$,
then ``what you see is what you get'':
predictors $h$ with low training loss $L_Z(h)$ will actually generalize well (have small $L_\dist(h)$).

For an empirical risk minimizer $\hat h \in \argmin_{h \in \cH} L_Z(h)$,
we further have for all $h^* \in \cH$ that
\[
    L_Z(\hat h) \le L_Z(h^*) = L_\dist(h^*) + \left( L_Z(h^*) - L_\dist(h^*) \right)
.\]
For any fixed $h^*$,
we can apply a simple concentration inequality\footnote{%
    Many sources instead bound $\sup_{h \in \cH} \abs{L_\dist(h) - L_Z(h)}$, and so do not need separate control of $L_Z(h^*) - L_\dist(h^*)$.
    While perhaps conceptually simpler,
    the description here makes clear that we only need to upper-bound $L_Z - L_\dist$ for $h^*$,
    while we need to upper-bound $L_\dist - L_Z$,
    or equivalently lower-bound $L_Z$,
    for all hypotheses.
    Thus asking for an upper bound everywhere is potentially wasteful.
    While with Rademacher complexity the gap in the bound is only a small constant,
    some methods \citep[e.g.][]{small-ball} can give much tighter lower bounds than upper bounds.
    The hardness result of \citet{vaish:uniform-failures},
    used to argue that ``uniform convergence may be unable to explain generalization in deep learning,''
    is also fundamentally about cases where $L_Z - L_\dist$ is high;
    no such result is possible in the one-sided case \citep[footnote 5]{zhou:uniform-interpolation}.
}
(such as Hoeffding's)
to show that $L_Z(h^*) - L_\dist(h^*)$ is small.
Then, if we choose that fixed $h^*$ to have loss 
(arbitrarily close to) $\inf_{h \in \cH} L_\dist(h)$,
we know $L_Z(\hat h)$
will not be much bigger than $\inf_{h \in \cH} L_\dist(h)$.
If we also have uniform convergence,
meaning the last term in \eqref{eq:uniform-conv} is known to be small,
then $\hat h$ will be nearly as good as the optimal predictor from $\cH$.

It is worth noting that uniform convergence \eqref{eq:uniform-conv}
gives only an upper bound on the generalization performance.
While uniform convergence occurs iff the problem is uniformly learnable
for binary classification \citep[e.g.][Theorem 6.7]{ssbd},
the same is not true in more general settings.
There are explicit constructions, for example, in
multiclass classification \citep[Section 6]{natarajan:some-results},
mean estimation with missing data \citep[Exercise 13.2]{ssbd},
and high-dimensional linear classification \citep[Section 3.1]{vaish:uniform-failures} and regression \citep[Sections 3.1 and 3.2]{zhou:uniform-interpolation},
where learning is possible but uniform convergence does not hold.\footnote{%
    The examples of \textcites[Section 6, example following Theorem 8]{natarajan:some-results}[Exercise 13.2]{ssbd}[Section 3.1]{zhou:uniform-interpolation} apply to one-sided uniform convergence,
    as well as two-sided.
    While \citeauthor{zhou:uniform-interpolation} left the one-sided failure as a conjecture, that conjecture is true when $d_S \ge 1$, as we now show.
    \endgraf
    The proof is identical to the two-sided case once we know
    $\E \lambda_{\max}(\Sigma - \hat\Sigma)$ from their Proposition B.2 (which they denote $\rho$)
    is $\Omega(\sqrt{\lambda_n / n})$.
    To show this, fix a unit vector $u \in \reals^{d_S}$
    and let $q = - X_J\tp X_S u / \norm{X_J\tp X_S u} \in \reals^{d_J}$.
    Consider representing $\Sigma - \hat\Sigma$
    in a basis beginning with
    the orthonormal vectors $(u, 0)$ and $(0, q)$;
    the upper-left $2 \times 2$ submatrix in that basis is
    \[
    A :=
    \begin{bmatrix} u\tp \bigl( I_{d_S} - \frac1n X_S\tp X_S \bigr) u & u\tp \bigl(-\frac1n X_S\tp X_J \bigr) q \\ q\tp \bigl(-\frac1n X_J\tp X_S \bigr) u  & q\tp \left( \frac{\lambda_n}{d_J} I_{d_J} - \frac1n X_J\tp X_J \right) q \end{bmatrix}
    = \begin{bmatrix} 1 - \frac1n \norm{X_S u}^2 & \frac1n \norm{X_J\tp X_S u} \\ \frac1n \norm{X_J\tp X_S u} & \frac{\lambda_n}{d_J} - \frac1n \norm{X_J q}^2 \end{bmatrix}
    .\]
    The largest eigenvalue of $\Sigma - \hat\Sigma$
    is at least as large as $\lambda_{\max}(A)$.
    Let $Y = \norm{X_S u}$.
    Since $X_J X_J\tp \to \lambda_n I$ as $d_J \to \infty$,
    we have $\norm{X_J\tp X_S u}^2 \to \lambda_n Y^2$
    and $\norm{X_J q}^2 = \norm{X_J X_J\tp X_S u}^2 / \norm{X_J\tp X_S u}^2 \to \lambda_n$,
    with all convergences almost sure;
    thus $A \to
    M := \begin{bmatrix} 1 - \frac1n Y^2 & \frac1n \sqrt{\lambda_n} Y \\ \frac1n \sqrt{\lambda_n} Y & -\frac1n \lambda_n \end{bmatrix}$.
    As
    $\lambda_{\max}\left(\begin{bmatrix} \alpha & \beta \\ \beta & \gamma \end{bmatrix}\right) = \frac12 \left(\alpha + \gamma + \sqrt{(\alpha - \gamma)^2 + 4 \beta^2}\right) \ge \frac12(\alpha + \gamma) + \abs\beta$,
    we have
    \[
        \lambda_{\max}(\Sigma - \hat\Sigma)
        \ge \lambda_{\max}(A)
        \to \lambda_{\max}(M)
        \ge \frac12 \left(1 - \frac1n Y^2 \right) - \frac{\lambda_n}{2n} + \frac{\sqrt{\lambda_n}}{n} Y
    .\]
    Since $u$ is a unit vector and $X_S$ is standard normal,
    $Y^2$ is $\chi^2(n)$;
    hence $(1 - \frac1n Y^2)$ has mean zero.
    $Y$ is $\chi(n)$,
    with mean
    $\sqrt{n} + O(1/\sqrt{n})$;
    thus $\E \sqrt{\lambda_n} Y / n
    = \sqrt{\frac{\lambda_n}{n}} \left( 1 + O(1 / n) \right)$.
    Recalling that in the problem setup it was assumed $\lambda_n = o(n)$,
    we have shown that
    \[
        \E \lambda_{\max}(M)
        \ge
        \sqrt{\frac{\lambda_n}{n}} \left(
            - \frac12 \sqrt{\frac{\lambda_n}{n}}
            + 1
            + O\left( \frac{1}{n} \right)
        \right)
        = \Omega\left( \sqrt{\frac{\lambda_n}{n}} \right)
    .\]
    Following the proof of their Theorem 3.2, we thus have
    $\displaystyle \lim_{n \to \infty} \lim_{d_J \to \infty} \E \sup_{\norm w \le \norm{\hat{w}_\mathrm{MN}}} \left[L_\dist(w) - L_S(w)\right] = \infty$.
}

\subsection{Rademacher complexity}
To show uniform convergence,
we can use
\emph{Rademacher complexity}.
Letting $\sigma_1, \dots, \sigma_m$ independently follow a Rademacher distribution
$\operatorname{Uniform}(\{ -1, +1 \})$,
the (expected) Rademacher complexity of a set of real-valued functions $\cF$ is
\[
    \Rad_{\dist^m}(\cF)
    :=
    \E_{z_1, \dots, z_m \sim \dist} \E_{\sigma_1, \dots, \sigma_m}\left[
        \sup_{f \in \cF}
        \frac1m \sum_{i=1}^m \sigma_i f(z_i)
    \right]
.\]
This definition arises from a technique called symmetrization,
which yields \citep[Proposition 4.2]{ltfp}
\begin{equation} \label{eq:oneside-rad}
    \E_{Z \sim \dist^m} \sup_{f \in \cF} \Bigl[
        \E_{z \sim \dist}[f(z)]
        - \frac{1}{m} \sum_{i=1}^m f(z_i)
    \Bigr]
    \le 2 \Rad_{\dist^m}(\cF)
.\end{equation}

One can also use ``desymmetrization'' to show a closely related lower bound,
although the form is slightly more complex
\parencites[Section 2.2]{koltchinskii:local-rad}[Proposition 4.12]{wainwright}.

Plugging in $\cF = \{ z \mapsto \ell(h, z) : h \in \cH \}$
gives exactly a bound on $\E_{Z \sim \dist^m} \sup_{h \in \cH} [ L_\dist(h) - L_Z(h) ]$.
We can further obtain high-probability bounds, if $\ell$ is bounded,
by McDiarmid's inequality \citep[Section 4.4.1]{ltfp}.

Finally, we often bound
the complexity of a loss class
using the complexity of the hypothesis class.
If $h(x) \in \reals$
and $\abs{\ell(h, (x, y)) - \ell(h', (x, y))} \le G \abs{h(x) - h'(x)}$,
then Talagrand's contraction lemma \citep[Proposition 4.3]{ltfp}
implies that
$\Rad(\cF) \le G \Rad(\cH)$;
thus
\[
    \E_{Z \sim \dist^m} \sup_{h \in \cH} [L_\dist(h) - L_Z(h)]
    \le 2 G \Rad(\cH)
.\]

Let's turn for a moment to the problem of learning a linear function in some fixed feature space,
$\cH_\cW = \{ x \mapsto \inner{w}{\phi(x)} : w \in \cW \}$.
Choosing $\phi(x) = x$
covers standard linear predictors.
Choosing $\phi$ to map into some Hilbert space
(here assumed to be over the reals for simplicity)
corresponds to learning in a reproducing kernel Hilbert space (RKHS).

The usual bound for the Rademacher complexity of $\cH_\cW$ is
\begin{equation} \label{eq:linear-class-rad}
\begin{aligned}
     \Rad_{\dist^m}(\cH_\cW)
  &= \E \sup_{w \in \cW} \frac1m \sum_{i=1}^m \sigma_i \inner{w}{\phi(x_i)}
\\&= \E \sup_{w \in \cW} \inner*{\frac1m \sum_{i=1}^m \sigma_i \phi(x_i)}{w}
\\&\le \E \norm[\Big]{\frac1m \sum_{i=1}^m \sigma_i \phi(x_i)} 
    \left(\sup_{w \in \cW} \norm{w} \right)
    &&\text{by Cauchy-Schwarz}
\\&\le \frac{1}{\sqrt m} \sqrt{\E_{x \sim \dist} \norm{\phi(x)}^2} \left(\sup_{w \in \cW} \norm{w} \right)
    &&\text{by Jensen and $\E [\sigma_i \sigma_j] = \indic(i = j)$}
.\end{aligned}
\end{equation}
Since $\sup_{w \in \cW} \norm{w}$ appears so naturally here,
it is tempting to think that it plays a vital role:
low-norm functions are easy to learn,
while high-norm functions are hard.
This almost seems like a ``built-in'' inductive bias for learning functions from $\cH_\cW$.

This intuition is incorrect.

\begin{proposition}[{\citealp[Exercise 4.9]{ltfp}}] \label{thm:translate-rad}
    Let $\cH$ be any set of functions $\mathcal X \to \reals$,
    $f : \mathcal X \to \reals$ any function,
    and write $\cH + \{ f \} := \{ x \mapsto h(x) + f(x) : h \in \cH \}$.
    Then $\Rad_{\dist^m}(\cH + \{f\}) = \Rad_{\dist^m}(\cH)$.
\end{proposition}
\begin{proof}
    Since $f$ does not interact with $h$
    and $\E \sigma_i = 0$,
    we have
    \[
        \E\left[ \sup_{h' \in \cH + \{f\}}
        \sum_i \sigma_i h'(x_i) \right]
        = 
        \E \sup_{h \in \cH} \left[
        \sum_i \sigma_i h(x_i)
        + \sum_i \sigma_i f(x_i)
        \right]
        = 
        \E \sup_{h \in \cH} \left[
        \sum_i \sigma_i h(x_i)
        \right]
    .\qedhere\]
\end{proof}

In particular, translating the class is irrelevant: for any arbitrary $w_0$,
we have
$\Rad_{\dist^m}(\cH_\cW) = \Rad_{\dist^m}(\cH_{\cW - \{w_0\}})$.
Thus,
\[
    \Rad_{\dist^m}(\cH_\cW)
    = \inf_{w_0} \Rad_{\dist^m}(\cH_{\cW - \{w_0\}})
    \le \frac{1}{\sqrt m} \, \sqrt{\E \norm{\phi(x)}^2} \,
    \left(\inf_{w_0} \sup_{w \in \cW} \norm{w - w_0}\right)
.\]
That is, it is not the maximum norm but the \emph{radius} of $\cW$ that matters.
While our previous upper bound \eqref{eq:linear-class-rad} did indeed prefer low-norm predictors to high-norm ones,
that was an artifact of our proof.
In fact, we could have subtracted $w_0$ in \eqref{eq:linear-class-rad} in the first place,
``naturally'' yielding the radius.

This is exactly the picture of \cref{fig:circles}:
the Rademacher bound only cares about the size of each hypothesis class.
It does \emph{not} tell you which functions to prefer.
By \cref{thm:translate-rad},
this is not specific to linear classes;
\emph{any} function class can be translated by \emph{any} function without modifying the Rademacher complexity.

\subsection{Combinatorial dimensions for classification}
For binary classifiers,
where both upper and lower bounds for learnability are controlled by the VC dimension \parencites{vapnik-chervonenkis-1968}[Theorem 6.7]{ssbd},
an exactly analogous result holds.
There is no such thing as an ``inherently simple'' binary classifier;
it only matters how flexible the hypothesis class is.

\begin{proposition}[{\citealp[Exercise 3.25]{mrt}}] \label{thm:vc-xor}
    Define the operation
    $a \xor b := \indic(a \ne b)$ for $a, b \in \{0, 1\}$.
    Write $\cH \xor f := \{ x \mapsto h(x) \xor f(x) : h \in \cH \}$.
    Then, $\VCdim(\cH \xor f) = \VCdim(\cH)$.
\end{proposition}
\begin{proof}
    Let $h_1, \dots, h_N \in \cH$
    be hypotheses
    achieving the vectors of labels $y_1, \dots, y_N \in \{0, 1\}^{\abs X}$
    for the points in a set $X$.
    Write $y$ for the vector of labels achieved by $f$ on $X$.
    Then $h_1 \xor f, \dots, h_N \xor f$
    achieve the labels $y_1 \xor y, \dots, y_N \xor y$,
    with $\xor$ elementwise.
    As xor is a bijection,
    the number of labellings in each set is the same;
    thus $\cH$ shatters $X$ iff $\cH \xor f$ does.
\end{proof}

The same is true for multiclass classification,
where sample complexity is controlled by the Natarajan and DS dimensions \citep{cohen:agnostic-multiclass}.
\begin{proposition}
    Let $\cH$ be a class of functions from $\cX$ to $\cY \subseteq \nats$.
    For each $x \in \cX$, let $\pi_x$ be an arbitrary bijection on $\cY$,
    and let $\pi \circ \cH = \{ x \mapsto \pi_x(h(x)) : h \in \cH \}$.
    Then the Natarajan dimensions \citep{natarajan-dim} of $\cH$ and $\pi \circ \cH$ are equal;
    the same is true for the DS dimension \citep{ds-dim}.
\end{proposition}
\begin{proof}
    The proof proceeds exactly as that of \cref{thm:vc-xor},
    but using the concepts of Natarajan shattering
    or DS shattering instead of VC shattering.
\end{proof}

\subsection{PAC-Bayes}
Another major technique for proving generalization
is PAC-Bayes bounds, introduced by \citet{McA99};
\citet{alquier:pac-bayes} gives a recent overview.
In these bounds, there is a \emph{prior} distribution $\pi$
and a \emph{posterior} distribution $\rho$ over predictors,
although these need not be related by the rules of Bayesian inference.
One common example is the following Catoni bound,
obtained by combining a change-of-measure inequality with Hoeffding's inequality.
\begin{proposition}[{\parencites{Catoni2007}[Theorem 2.1]{alquier:pac-bayes}}]
    Let $\pi$ be a probability measure on $\cH$ independent of $Z$,
    and $\rho$ range over all probability measures on $\cH$.
    For a loss bounded in $[0, C]$,
    any fixed $\lambda > 0$,
    and any fixed $\delta \in (0, 1)$,
    \[
        \Pr_{Z \sim \dist^m}\left( 
        \forall \rho,\quad
        \E_{h \sim \rho}[L_\dist(h) - L_Z(h)]
        \le
          \frac{\lambda C^2}{8 m}
        + \frac{\kl{\rho}{\pi} + \log\frac1\delta}{\lambda}
        \right) \ge 1 - \delta
    .\]
\end{proposition}
As the KL divergence is invariant to (measurable) invertible transformations,
changing the prior and the posterior in the same way
will not affect this bound on the generalization gap at all.
Other bounds than the above have roughly the same structure.

\subsection{Stability}
Probably the only other widely-used technique for proving generalization bounds is \emph{algorithmic stability};
this can apply outside of settings with uniform convergence.
For instance,
the most common definition is uniform stability:
\begin{definition}[\citep{uniform-stability,uniform-stability-random}]
    When $Z = (z_1, \dots, z_m)$,
    let $Z_{\lnot i}$ denote the removal of $z_i$, $(z_1, \dots, z_{i-1}, z_{i+1}, \dots, z_m)$.
    A (possibly randomized) learning algorithm $\mathcal A$ is
    \emph{$\beta(m)$-uniformly stable} if for all $m \ge 1$,
    \[
        \sup_{\substack{Z \in \mathcal Z^m, \, i \in [m] \\ z \in \mathcal Z}}
        \abs{\E_{\mathcal A} \ell(\mathcal A(Z_{\lnot i}), z) - \E_{\mathcal A} \ell(\mathcal A(Z), z)}
        \le \beta(m)
    .\]
    We say an algorithm is \emph{uniformly stable} if $\lim_{m \to \infty} \beta(m) = 0$.
\end{definition}
Uniformly stable algorithms bear guarantees on the random $\E_{\mathcal A} \left[L_\dist(\mathcal A(Z)) - L_Z(\mathcal A(Z))\right]$ in terms of $\beta(m)$ \citep{uniform-stability,uniform-stability-random,sharper-uniform-stability}.
Once again, nothing in this definition privileges any particular form of hypothesis; it only matters how stable the algorithm's output is to small changes in the training set.
The same is true for various other notions of stability \citep{average-stability}.

\section{Epistemic Boltzmann brains} \label{sec:boltmzann-brains}

Cosmology is the branch of physics that studies the large-scale structure of the Universe. One criterion for evaluating a proposed model of the Universe is whether it predicts an overabundance of so-called \emph{Boltzmann brains}. These are observers that arise from chance fluctuations in a thermal background, instantiating a brief moment of subjective experience -- complete with illusory memories of an orderly past -- before dissipating back into equilibrium. In models where Boltzmann brains (BBs) vastly outnumber ordinary observers (OOs) like ourselves, it is argued that by simple counting, we must conclude that we are almost certainly BBs. Since this conclusion would completely undermine the scientific method, it is argued that we should reject such models \citep{barrow1987anthropic,de2010boltzmann,carroll2020boltzmann}.

In this section, we argue a contrary view: the problem of BBs is not specific to particular models of the cosmos, but is in fact a kind of No Free Lunch theorem. Even if our Universe contains no BBs at all, \textbf{without an inductive bias, we are forced to conclude that we are probably Boltzmann brains}. On the other hand, \citet{hutter2010complete,muller2020law,muller2026algorithmic} show that with a universal inductive prior, even observers who know their Universe is full of BBs will view themselves as OOs, and make ordinary inferences. Hence, there is no reason to judge a model by its prevalence of BBs.

We make four versions of the argument, first from a position of total epistemic ignorance, adding more generous assumptions at each subsequent iteration. 

\subsection{Total epistemic ignorance}
Ideally, we should be able to make inferences about our surrounding world without any prior assumptions. Through our biological senses, we take a sequence of observations, which we hope to extrapolate into predictions about future observations. Unfortunately, without an inductive bias, the No Free Lunch Theorem applies exactly as stated in \Cref{lem:nfl}. That is, no matter what observations we accumulate, it remains equally plausible to receive any sequence of future observations. If we start with a uniform prior, then with overwhelmingly high subjective probability, we should expect any apparent structure in our past to dissolve immediately, just as for a BB.

\subsection{Knowledge of dynamics}
Now we add the assumption that we know the dynamical laws of the physical Universe, but start with no knowledge of its particular state. \citet{wolpert2025disentangling} argue that concluding we are BBs from such an assumption would be circular and contradictory, because a BB cannot scientifically infer the laws of physics. Nonetheless, it should not \emph{hurt} us to assume we are endowed with knowledge of the laws, as if by divine inspiration.

Without an inductive bias, we may take a maximum entropy prior (which we assume is well-defined) over the possible states of the Universe. If we do, then our subjectively held Bayesian belief is identical to the large-scale equilibrium distribution of the Universe, known as \emph{heat death}. By the second law of thermodynamics, the time-evolution of our belief distribution remains at maximum entropy at all future times. Thus, even if the Universe is objectively quite orderly, our belief remains equivalent to a BB at heat death, unable to make useful inferences about the state of our surroundings. Before expanding on some details of this argument, we add an even more generous assumption.

\subsection{Control of dynamics}
Now we add the assumption that we can choose the dynamical laws at will, subject to obeying the second law of thermodynamics, i.e.\ that entropy cannot decrease. In other words, while the previous argument assumed we were divinely inspired by knowledge of physics, here we assume that we are divinely endowed with the ability to \emph{control} physics as we see fit. We will see that, even granting such unrealistic powers, it remains impossible to make useful inferences about our surroundings.

At first, this may seem absurd. If we want to learn the state of our environment, surely we can just look at it and record its value to our memory; or failing that, we apply our physics-bending powers to \emph{make} the environment match our beliefs! The issue is that the second law of thermodynamics does not play well with maximum entropy priors.

To illustrate, consider an ordered pair $(m_t,e_t)$ representing the state of our memory and environment, respectively, at time $t$. For simplicity, we model the states as discrete. In order to measure $e_t$, we might like to implement the assignment operation $m_{t+1}:= e_t$, which can be represented as a map
\[(m,\,e)\mapsto (e,\,e).\]

However, this map is irreversible, decreasing the entropy of the joint system. Nature forbids irreversible operations in isolation; in reality, such measurements are only possible by dumping the excess entropy into a third system that is known to have low entropy \citep{bennett1982thermodynamics}. One valid implementation first initializes the memory to zero by dissipating entropy into a heat bath, after which we can implement the reversible (i.e.\ injective) map
\[(0,\,e)\mapsto (e,\,e).\]

Unfortunately, a maximum entropy prior assigns equal probability to every possible joint state of the Universe, so we have no low-entropy heat bath to rely on. Since every transformation that obeys the second law necessarily preserves the maximum entropy distribution, our belief about the state of the Universe can never change. None of our manipulations would enable us to be confident that we have set the memory to zero, let alone infer the state of anything else.

\subsection{Control of dynamics \emph{and} initialized memory}
It seems strange to imagine not knowing the state of our own memory, since by definition the memory should be used to record our beliefs. In this final version of the argument, we start with some memory of a fixed size, say 1 GB, containing known initial values, a large fraction of which are zeros. Our prior is maximum entropy on everything in the Universe except for this memory, resulting in a total entropy that is 1 GB less than the maximum possible.

In this case, we can indeed perform some measurements, after which our memory's state becomes correlated with the environment. However, by the second law of thermodynamics, our belief distribution will never be more than 1 GB away from heat death. Converted to physical units, this gap is less than $\SI{e-13}{\joule\per\kelvin}$ \citep{frank2005indefinite,Bedard2025Temperature}, so our observations can never assure us of large-scale structure sufficient to extract a macroscopically useful quantity of work. Normally, we make large-scale inferences by generalizing from smaller amounts of data; in contrast, the near-maximum entropy prior forces us to conclude the smallest possible correlation consistent with our measurements. Once again, we are effectively trapped as epistemic BBs.

\citet{wolpert2025disentangling} propose a solution, effectively arguing that our prior belief should correspond to a maximally uncertain distribution on the trajectory of the Universe, conditional on some coarse-grained information from two snapshots: the Big Bang and the present day. We take the position that their prior is unjustified, since we cannot directly observe those snapshots. Instead, scientific inferences about the Universe are more readily justified via (approximate) universal induction on real observations \citep{hutter2010complete,muller2020law,muller2026algorithmic}.

\subsection{Accessibility in physics}

In stochastic thermodynamics, a physical system is modelled as having a randomly distributed state \citep{seifert2025}. It can do useful work only when its state has low Shannon entropy. \citet{bennett1982thermodynamics} argued that a deterministically prepared computer memory can be used to do work when its value is algorithmically simple (such as the first trillion digits of $\pi$), and not when it is algorithmically random \citep{MartinLof1966}, despite having zero Shannon entropy in both cases. He concluded that thermodynamic entropy should be defined for individual states, without reference to probability distributions, using the Kolmogorov complexity. Indeed, it seems that thermodynamic entropy can be understood relative to computational primitives that are made accessible by the laws of physics \citep{ebtekar2025foundations}.

The intuitive idea follows from universal induction. After observing a small number of digits of $\pi$, we can discover a short program that prints them, and thereafter the rest of the memory becomes predictable. The program length is measured in terms of whatever means of computation are accessible in the physical world \citep{deutsch2013computation,ozkural2015ultimate,janzing2018does,ebtekar2025foundations}.

In contrast, we can only learn an algorithmically random value by observing and measuring all of it. In that case, we do learn the memory's value, and can even clear it. However, there is no free lunch, as we would fill our biological memory (or measurement apparatus) with the original contents of the computer memory, which are thermodynamically costly to erase \citep{landauer1961irreversibility,maruyama2009colloquium}.

\end{document}